%% file: struct_decomposition_hal.tex
\documentclass[11pt]{article}

\newenvironment{proof}{\par\noindent{\bf Proof\ }}{\hfill\BlackBox\\[2mm]}
\newtheorem{lemma}{Lemma}

\newtheorem{proposition}{Proposition}

\input macros.tex

\title{Convex relaxations of structured matrix factorizations}

\author{Francis Bach \\
Sierra Project-team\\
D\'epartement d'Informatique de l'Ecole Normale Sup\'erieure\\
Paris, France
\\{\tt francis.bach@ens.fr}
        }

\begin{document}

\maketitle

\begin{abstract}
We consider the factorization of a rectangular matrix $X  $ into a positive linear combination of rank-one factors of the form $u v^\top$, where $u$ and $v$ belongs to certain sets $\U$ and $\V$, that may encode specific structures regarding the factors, such as positivity or sparsity. In this paper, we show that computing the optimal decomposition is equivalent to computing a certain gauge function of $X$ and we provide a detailed analysis of these gauge functions and their polars. Since these gauge functions are typically  hard to compute, we present semi-definite relaxations and several algorithms that may recover approximate decompositions with approximation guarantees. We illustrate our results with simulations on finding decompositions with elements in $\{0,1\}$. 
As side contributions, we present a detailed analysis of variational quadratic representations of norms as well as a new iterative basis pursuit algorithm that can deal with inexact first-order oracles.
\end{abstract}

 \section{Introduction}

 Structured matrix factorization has many applications in various areas of science engineering, i.e., clustering and principal component analysis~\cite{duda2012pattern,murphy2012machine}, source separation~\cite{lee1999learning,fevotte2009nonnegative}, signal processing~\cite{Aharon2006}, machine learning~\cite{Mairal2010}, and all domains where  reduced representations are desired.
 
 Without any structure, traditional principal component analysis may be solved exactly in polynomial time through a singular value decomposition. However, adding additional structure on the components $U$ and $V$ of the factorization of $X = UV^\top$ (e.g., sparsity, non-negativity or discreteness) is most often done through alternating minimization (with respect to $U$ and $V$). While all steps are usually done through convex optimization, the problem is not jointly convex, and there are typically multiple local minima, and algorithms usually come with no convergence guarantees. For example, in presence of positivity constraints, the problem of non-negative matrix factorization (NMF) may not be solved in polynomial-time in general, and most algorithms for NMF (e.g.,~\cite{lee1999learning}) perform a form a block coordinate descent with no guarantees (see hardness results and particular situations of actual solvability in~\cite{arora2012computing,recht2012factoring}). In this paper, we follow a convex relaxation approach.

We impose some structure on each column of $U$ and $V$ and consider a general convex framework which amounts to computing a certain gauge function (such as a norm) at $X$, from which the decomposition may be obtained (for example, the nuclear norm leads to the usual singular value decomposition). This convex framework corresponds to removing any rank constraint on the decomposition and has appeared under various forms in the literature, as summing norms~\cite{jameson1987summing}, decomposition norms~\cite{bach2008convex} or (a special case of) atomic norms~\cite{chandrasekaran2012convex}. This is presented in details (equivalent representations,  rotation-invariant cases,  weighted nuclear norm formulations) in \mysec{decomposition}. An interesting aspect is that the gauge functions we consider are polar to generalizations of matrix norms, which are commonly used in many areas of applied mathematics, in particular in robust optimization and control~\cite{ben2009robust}.

The statistical and recovery properties of these norms and their relaxations have been studied in several contexts~\cite{chandrasekaran2012convex,recht2010guaranteed}; in this paper, we focus on optimization aspects.
The convex framework we introduce in~\mysec{decomposition} only lead to polynomial-time algorithms in few situations (e.g., the nuclear norm based on the singular value decomposition). In \mysec{relaxation}, we consider computable additional relaxations based on semi-definite programming. These may be used to compute the related gauge functions as well as their polar, with constant-factor approximation guarantees in some cases (see \mysec{guarantees}). The first setting where one can get dimension-independent guarantees has already been studied by~\cite{nesterov1998semidefinite,tal} and corresponds to gauge functions that have variationl  diagonal representations. We also consider a more general setting with dimension-dependent bounds.

A key practical problem is to obtain not only a lower-bound on the value of the gauge function, but also an \emph{explicit} decomposition which preserves the approximation guarantees. We present in \mysec{algorithm}  iterative conditional gradient algorithms and their analysis, which extend existing results in several ways: (a) we obtain convergence guarantees even when the polar gauge function may be approximately computed with a \emph{multiplicative} approximatio ratio---earlier work~\cite{jaggi} considers only additive approximations, (b) following~\cite{SGCG,zaidcg,zhang2012accelerated}, they may be applied to penalized versions of the problem, i.e., to solve a generalized basis pursuit problem, (c) under some additional assumptions, they may find approximate decompositions of $X$ that converge linearly.

Finally, our framework relies on variational representations of gauge functions and in particular norms as maxima and minima of quadratic functions. These representations have been already used in several contexts (machine learning, signal processing, optimization, see, e.g.,~\cite{fot} and references therein). In this paper, we provide in \mysec{quadratic} a thorough analysis of these decompositions (minimal and maximal representations, duality between lower and upper bounds, sufficient and necessary conditions for diagonal   or rotation-invariant representations).

 \paragraph{Notation.} Given a positive integer $d$, we denote by $\mathcal{S}_d$ the vector space of symmetric matrices, and  by $\S_d^+$ the subset of positive-semidefinite matrices. For $x \in \rb^d$ and $ p \in [1,+ \infty]$, $\|x\|_p$ denotes the $\ell_p$-norm of $x$, while for a matrix $X \in \rb^{n \times d}$, $\| X\|_p$ denotes the $\ell_p$-norm of $X$, seen as a vector. That is, if ${\rm vec}(X) \in \rb^{n d}$ denotes the vector obtained by stacking the columns of $X$, $\| X\|_p = \| {\rm vec}(X)\|_p$. The Frobenius norm of $X$ is denoted $\| X\|_F = \sqrt{ \tr X^\top X} = \| {\rm vec}(X)\|_2$, the nuclear norm (a.k.a.~the trace norm) is denoted as $\| X\|_\ast$, and is equal to the sum of the singular values of $X$, while the operator norm (largest singular value of $X$) is denoted $\| X\|_{\rm op}$. Finally, $1_d$ denotes the vector in $\rb^d$ with all components equal to one.

  \section{Review of gauge function theory}
 \label{sec:gauge}
 In this section, we present relevant concepts and results from convex analysis.
 These tools are needed because the type of structure we want to impose go beyond what can be characterized by norms (such as positivity).
  See~\cite{rockafellar97,borwein2006caa} for more details on gauge functions and their properties.
  
  \paragraph{Gauge functions.} 
  Given a closed convex set $\C \subset \rb^d$, the gauge function $\gamma_\C$ is the function
   $$
 \gamma_\C(x) = \inf \{ \lambda \geqslant 0, \ x \in \lambda \C \}.
 $$
 The domain ${\rm dom}(\gamma_\C)$ of $\gamma_\C$ is the cone generated by $\C$, i.e., $\rb_+ \C$ (that is, $\gamma_\C(x) < +\infty$ if and only if $x \in \rb_+ \C$).  The function $\gamma_\C$ is equivalently defined as the homogeneized version of the indicator function $I_\C$ (with values $0$ on $\C$ and $+\infty$ on its complement), i.e., 
$
\gamma_\C(x) = \inf_{\lambda \geqslant 0} \lambda I_\C\Big(\frac{x}{\lambda} \Big).
$
From this interpretation, $\gamma_\C$ is therefore a convex function. Moreover, it is positively homogeneous and has non-negative values. Conversely, any function $\gamma$ which satisfies these three properties is the gauge function of the set $\{ x \in \rb^d, \ \gamma(x) \leqslant 1\}$.

Several closed convex sets $\C$ lead to the same gauge function. However the unique closed convex set containing the origin is $\{ x \in \rb^d, \ \gamma_\C(x) \leqslant 1\}$. In general, we have for any closed convex set $\C$, $\{ x \in \rb^d, \ \gamma_\C(x) \leqslant 1\} = {\rm hull}( \C \cup \{0\})$.

  Classical examples are norms, which are gauge functions coming from their unit balls: norms are gauge functions $\gamma$ which (a) have a full domain, (b) are such that $\gamma(x) = 0 \Leftrightarrow x = 0$, and (c) are even, which corresponds to sets $\C$  which (a) have 0 in its interior, (b) are compact and (c) centrally symmetric. In general, the set $\C$ might neither be compact nor centrally symmetric, for example, when $\C = \{ x \in \rb_d^+, \ 1_d^\top x \leqslant  1\}$. Moreover, a gauge function may take infinite values (such as in the previous case, for any vector with a strictly negative component).

  \paragraph{Polar sets and functions.}
  Given any set $\C$ (not necessarily convex), the polar of~$\C$ is the set $\C^\circ$ defined as
  $$
  \C^\circ = \{ y \in \rb^d, \ \forall x \in \C, \ x^\top y \leqslant 1\}.
  $$
  It is always closed and convex. Moreover, the polar of $\C$ is equal to the polar of the closure of  ${\rm hull}( \C \cup \{0\})$.
  
  When $\C$ is the unit ball of a norm $\Omega$, $\C^\circ$ is the unit ball of the dual norm which we denote $\Omega^\circ$ (instead of the usual definition $\Omega^\ast$, because the Fenchel conjugate of $\Omega$ is not the dual norm, but the indicator function of the dual unit ball).
  
    If $\C$ is a closed convex set containing the origin, then $\C^{\circ \circ} = \C$---more generally, for any set~$\C$, $\C^{\circ \circ}$ is the closure of 
    ${\rm hull}( \C \cup \{0\})$.
The polarity is a one-to-one mapping from closed convex sets containing the origin to themselves. In this paper, we will also consider gauge functions associated with closed potentially non convex sets $\C$, in which case, we mean the gauge function associated to $\C^{\circ\circ} = {\rm hull}( \C \cup \{0\})$, i.e., $\gamma_\C = \gamma_{\C^{\circ\circ}}$.

  The Fenchel conjugate of $\gamma_\C$ is the indicator function of $\C^\circ$, i.e.,
  $
  \gamma_\C^\ast = I_{\C^\circ},
  $
  which is equivalent to $\gamma_\C = I_{\C^\circ}^\ast$, i.e., $\forall x \in \rb^d, \ \gamma_\C(x) = \sup_{y \in \C^\circ} x^\top y$.
  Given a gauge function $\gamma_\C$, we define its polar as the function $\gamma_\C^\circ$ given by
  $$
  \gamma_\C^\circ(y) = \inf \big\{ \lambda \geqslant 0, \  \forall x \in \rb^d, \ x^\top y \leqslant \lambda \gamma_\C(x) \big\}
  = \sup_{ x \in \rb^d} \frac{ x^\top y}{ \gamma_\C(x)},
  $$
  the last inequality being true only if $\gamma_\C(x) = 0 \Leftrightarrow x =0 $ (i.e., $\C$ compact).
  It turns out that 
  $
   \gamma_\C^\circ =  \gamma_{\C^\circ}.
  $
  This implies that $\gamma_{\C^\circ} = I_{ C^{\circ \circ} }^\ast$, i.e., $\gamma_{\C^\circ}(y) = \sup_{ x \in \C^{\circ \circ}} x^\top y  = \sup_{ x \in \C} (x^\top y)_+$. 
  For example, the polar of a norm is its dual norm. We have for all $x,y \in \rb^d$, the inequality that is well known for forms: $x^\top y \leqslant \gamma_\C(x) \gamma_{\C^\circ}(y)$. Finally, the Fenchel-conjugate of $x \mapsto \frac{1}{2} \gamma_\C(x)^2$ is $y \mapsto \frac{1}{2} \gamma_\C^\circ(y)^2$.

\paragraph{Operations on gauge functions.} For two closed convex sets $\C$ and $\D$ containing the origin, then for all $x \in \rb^d$, $ \max \{ \gamma_\C(x),\gamma_\D(x) \} = \gamma_{\C \cap \D}(x)$. Another combination, is the ``inf-convolution'' of $\gamma_\C$ and $\gamma_\D$, i.e.,
$x \mapsto \inf_{x = y+ z} \gamma_\C(z) + \gamma_\D(y)$, which is equal to $\gamma_{ {\rm hull}(C \cup D) }$. Moreover, $\gamma_{\C \cap \D}^\circ = \gamma_{ {\rm hull}(C^\circ \cup D^\circ) }$, or equivalently, $ ( C \cap \D)^\circ =  {\rm hull}(C^\circ \cup D^\circ) $. 

\paragraph{Links with convex hulls.}

Given a compact set $\P$ and its compact convex hull $\C$ (for example, $\P$ might be the set of extreme points of $\C$), we have
$
\P^\circ = \C^\circ,
$
since maxima of linear functions on $\C$ or $\P$ are equal.
 An alternative definition of $\gamma_\C$ is then  
 $$  \gamma_\C(x) = \min \bigg\{ \sum_{i \in I} \eta_i, \ (\eta_i)_{i \in I} \in \rb_+^I, \ (x_i)_{i \in I} \in \P^I,   \ I \mbox{ finite}, \ x = \sum_{i \in I} \eta_i x_i \bigg\}. $$
 Moreover, in the definition above, by Caratheodory's theorem for cones, we may restrict the cardinality of $I$ to be less than or equal to $d$.

\section{Representations of gauge functions through quadratic functions}
\label{sec:quadratic}

In this section, we consider a closed set $\U$ that contains the origin in its convex hull, and explore various representations of the gauge function $\gamma_\U = \gamma_{\U^{\circ\circ}}$ as a maxima or minima of quadratic functions.
This corresponds to respectively inner and outer approximations of the convex hull $\U^{\circ\circ}$ by ellipsoids. Note that unless otherwise stated, $\U$ might not be convex, might not be compact (i.e., $\gamma_\U(x)=0$ even if $x \neq 0$), $\U$ might not be centrally symmetric, and ${\rm dom}(\gamma_\U) = \rb_+ \U$ may be strictly included in $\rb^d$.  Since quadratic variational formulations have to be centrally symmetric, we consider symmetrized versions of gauge functions; we consider two ways of ``symmetrizing'' a gauge function: by intersecting $\U^{\circ\circ}$ and $-\U^{\circ\circ}$, which corresponds to the gauge function $\gamma_{ \U^{\circ\circ} \cap ( - \U^{\circ\circ})}(x) = \max\{ \gamma_\U(x),\gamma_\U(-x)\}$, or by intersecting $\U^{\circ}$ and $-\U^\circ$, which corresponds to $\gamma_{ \U \cup (-\U) } (x) = \gamma_{\U^\circ \cap (- \U^\circ)}^\circ(x) = \inf_{x = x_+ +  x_-} \big\{ \gamma_\U(x_+) + \gamma_\U(- x_-) \big\}$. These two gauge functions have domain ${\rm span}(\U)$, the vector space generated by $\U$. When $\U$ is centrally symmetric, then  these two gauge functions are equal to  $\gamma_\U$.

\subsection{Maxima of quadratic functions}
\label{sec:K}

We first consider \emph{closed convex} sets $\K \subset \S_d^+$ of symmetric positive definite matrices   such that  
\BEQ
\label{eq:K} \forall x \in {\rm span}(\U), \ \gamma_{ \U \cup (-\U) } (x)^2  = \max_{ M \in \, \K} x^\top M x . 
\EEQ
When $\U$ is symmetric, the variational formulation in \eq{K} leads to a representation of $\gamma_\U(x)^2$ as a convex function $I_\K^\ast(xx^\top)$ of $xx^\top$.
Note that in general, $\K$ is not unique. We now show that  there always exists a set $\K$ satisfying \eq{K}, and provide a description of the largest such set.
\begin{proposition}
\label{prop:Klargest}
Let $\P_\U = \{ xx^\top,  \ x \in \U\}$. Then   $\P_\U^\circ \cap \S_d^+$ is the larget closed convex  set $\K \subset \S_d^+$ of positive semidefinite matrices  such that
\eq{K} is satisfied.
\end{proposition}
\begin{proof}
We first show that $\P_\U^\circ \cap \S_d^+$ satisfies  \eq{K}. Let $x \in \rb^d$, we have
$
  \sup_{M \in  \P_\U^\circ} x^\top M x =  I_{\P_\U^\circ}^\ast(xx^\top)    =  \gamma_{\P_\U}(xx^\top) .
$
Since $\P_\U \subset \S_d^+$, $xx^\top$ may only be decomposed as the sum of matrices of the form $\lambda xx^\top$ and 
$\lambda (-x)(-x)^\top$, for $\lambda \geqslant 0 $. This implies that $\gamma_{\P_\U}(xx^\top)   =\min \{ \gamma_\U(x), \gamma_\U(-x) \} ^2$. Note that since $\P_\U^\circ$ does not only include positive semi-definite matrices, it is not incoherent that $x \mapsto  I_{\P_\U^\circ}^\ast(xx^\top)  = \sup_{M \in \P_\U^\circ} x^\top M x = \min \{ \gamma_\U(x), \gamma_\U(-x) \} ^2$ is not convex.

We now have
$
  \sup_{M \in  \P_\U^\circ \cap \S_d^+ } \sqrt{ x^\top M x } \leqslant   \sup_{M \in  \P_\U^\circ} \sqrt{x^\top M x} = \min \{ \gamma_\U(x), \gamma_\U(-x) \}$. Since this is a convex function of $x$, it must be less than its convex envelope, which is $ \gamma_{\U \cup (-\U)}(x)$ (since they have the same Fenchel conjugates).

Moreover, if $v \in \U^\circ \cap (-\U)^\circ$, then $vv^\top \in \P_\U^{\circ }$. This implies that
$  \sup_{M \in  \P_\U^\circ} x^\top M x \geqslant \sup_{ v \in \U^\circ \cap (-\U)^\circ} (v^\top x)^2 = I^\ast_{\U^\circ \cap (-\U)^\circ}(x)^2
= \gamma_{\U^\circ \cap (- \U^\circ)}^\circ(x)^2 = \gamma_{\U \cup (-\U)}(x)^2$. Thus  \eq{K} is indeed satisfied by $\P_\U^\circ \cap \S_d^+$. Finally,
if $\K$ satisfies \eq{K}, then we must have $\K \subset \P_\U^\circ$ by definition of polar sets, hence $\P_\U^\circ \cap \S_d^+$ is the largest. 
\end{proof}

\vspace*{.2500cm}

Note that the set $\P_\U^\circ \cap \S_d^+$ is equal to $\{ M \in \S_d^+, \ \forall u \in \U, u^\top M u \leqslant 1\}$---this representation was already considered in~\cite{fot} for norms.
In certain situations, the largest possible set is desirable (for example when deriving convex relaxations). In other situations (for example when using these representations for optimization), smallest sets are desirable. However, such a notion is not possible. Indeed, for $\gamma_\U = \| \cdot \|_2$, the sets $\K = \{ \idm \}$ and $\K = \{ M \in \S_d^+, \ \| M \|_F \leqslant 1\}$, are two possible sets, and thus there is no single smallest set. One possible small set is 
 the convex hull of the maximal elements of ${\P}_\U^\circ \cap \S_d^+$ (for the positive semi-definite order).

In the proof of Prop.~\ref{prop:Klargest}, we have introduced the gauge function $\gamma_{\P_\U}$.
We now provide a representation of $\gamma_{\P_\U}$ related to factorizations of positive semidefinite matrices (note that in the following proposition, we only assume that $\U$ is closed and contains 0 in its hull).
\begin{proposition}
\label{prop:UU}
Let $\P_\U = \{ xx^\top,  \ x \in \U\} \subset \S_d^+$. We have, for all positive semidefinite matrix $M \in \S_d^+$:
$$\gamma_{\P_\U}(M)= \inf_{r \geqslant 0}  \ \inf_{M  = \sum_{m=1}^r \! x_m x_m^\top } \sum_{m=1}^r \gamma_\U(x_m)^2
= \inf_{r \geqslant 0} \ \inf_{M  = \sum_{m=1}^r  \! \lambda_m x_m x_m^\top, \ x_m \in \U } \sum_{m=1}^r \lambda_m .$$
Moreover,  we may choose $r \leqslant d(d+1)/2$.
\end{proposition}
\begin{proof}
This is a direct application of the representation of gauge functions and and the property $\gamma_{\P_\U}(xx^\top)   =\min \{ \gamma_\U(x), \gamma_\U(-x) \} ^2$, that was shown in the proof of Prop.~\ref{prop:Klargest}.
\end{proof}

\vspace*{.2500cm}

Note that $\rb_+ \U = { \rm dom }(\gamma_\U)$ may not be equal to $\rb^d$, and that the domain of $\gamma_{\P_\U}$ may not be equal to $\S_d^+$, i.e., some positive matrices may not be decomposed as positive linear combinations of dyads $uu^\top$ obtained from elements $u$ of $\U$; for example, when $\U =\{ x \in \rb_+^d, \ 1_d^\top x =1\}$ is the simplex, then ${\rm dom}(\gamma_\U)$ is the set of completely positive matrices (see, e.g.,~\cite{berman2003completely}).

The last proposition provides a structured decomposition framework for positive semi-definite matrices, that will be considered for rectangular matrices in \mysec{decomposition}. Obtaining  explicitly the decomposition $M = UU^\top$ from $M$ may be done with the iterative algorithms presented in  \mysec{algorithms}. Note that by considering a representation of $U$ as $U = M^{1/2} S$ where $SS^\top = \idm$,  computing $\gamma_{\P_\U}$ may be seen as a factorization problem with two factors (which can then be used in alternating minimization procedures).

\subsection{Minima of quadratic functions}
\label{sec:L}

We now consider closed convex sets $\L \subset \S_d^+$ such that
\BEQ
\label{eq:L} \forall  x \in {\rm span}(\U) , \ \max\{\gamma_\U(x),\gamma_\U(-x)\}^2 = \inf_{ M \in \, \L} x^\top M^{-1} x . 
\EEQ
Here, we define $x^\top M^{-1} x$ as $x^\top M^{-1} x = \inf t $ such that 
$\Big( \!\begin{array}{cc} M & x \\ x^\top & t  \end{array} \! \Big) \succcurlyeq 0 
\Leftrightarrow   t M  \succcurlyeq xx^\top$. This implies that the value may be finite even when $M$ is not invertible.

When $\U$ is symmetric, the variational formulation in \eq{L} leads to a representation of $\gamma_\U(x)^2$ as a concave $ \inf_{M \in \L}  \tr M^{-1} xx^\top$ of $xx^\top$.
This is to be contrasted with the fact that it is also a \emph{convex} function of $xx^\top$ because of the representation discussed in \mysec{K}.
The two properties are in fact related through a duality argument:
\begin{proposition}
\label{prop:eqKL}
Let $\L \subset \S_d^+$ be a closed convex set. Then the following two properties are equivalent:
\BEAS
(a) &  \ \ & \forall x  \in {\rm span}(\U) , \ \inf_{ M \in \, \L} x^\top M^{-1} x =  \max \{  \gamma_\U (x),
  \gamma_\U (-x) \}^2  , \\
(b) &\ \  &   \forall y   \in {\rm span}(\U^\circ), \  \sup_{ M \in \, \L} y^\top M y=   \gamma_{\U \cup (-\U)} (y)^2.
\EEAS
This implies that the largest set $\L$ such that (a) is valid is $\P^\circ_{\U^\circ} \cap \S_d^+$   defined in Prop.~\ref{prop:Klargest}.
\end{proposition}
\begin{proof} The two   functions   $y \mapsto  \frac{1}{2}\gamma_{\U \cup (-\U)} (y)^2 =   \frac{1}{2} \inf_{y = y_+ +  y_-} \big\{ \gamma^\circ_\U(y_+) + \gamma^\circ_\U(- y_-) \big\}^2$
and $x \mapsto \frac{1}{2} \max \{  \gamma_\U (x),
  \gamma_\U (-x) \}^2$ are Fenchel-conjugate  to each other.
  We thus need to show that this is the same for $x \mapsto \frac{1}{2} \inf_{ M \in \, \L} x^\top M^{-1} x$ and $y \mapsto \frac{1}{2}\sup_{ M \in \, \L} y^\top M y$, which is straightforward to check. 
    Note that $\P_{\U^\circ} \subset \P^{\circ \circ}_{\U^\circ} \subset \P_\U^\circ \cap \S_d^+$, but that the reverse inclusions are typically not true (take, e.g., the unit $\ell_2$-ball).
\end{proof}

\vspace*{.2500cm}

  \paragraph{Duality between representations.}
  For a given gauge function $\gamma_\U$, there are two possible variational representations, as maxima or minima of quadratic functions.
  The two possible variational representations are linked to each other. Indeed,
  if $\K \subset \S_d^+$ is a closed convex set  such that for all $x \in \rb^d$, $\sup_{M \in \K} x^\top M x \leqslant  \gamma_{\U \cup (-\U)} (x)^2$ (i.e., we have a lower-bound which is convex in $xx^\top$), then
  $\K \subset \P_\U^\circ$, which implies by duality $\P_\U^{\circ \circ} \subset \K^\circ \cap \S_d^+$. Thus for any $y \in \rb^d$,
  $  I_{\K^\circ \cap \S_d^+ }^\ast(yy^\top) \geqslant I_{\P_\U}^\ast(yy^\top) = \max \{ \gamma_\U^\circ(y),\gamma_\U^\circ(-y)\}^2$.
  This then implies by Fenchel duality (i.e., using arguments from the proof of Prop.~\ref{prop:eqKL}) that for all $x \in \rb^d$,
  $
    \gamma_{\U \cup (-\U)} (x)^2
  \geqslant   \inf_{M \in \K^\circ \cap \S_d^+  } x^\top M^{-1} x ,  $
  i.e., we have a concave lower-bound based on $\K^\circ \cap \S_d^+$.
  
  Similarly, if we have a concave upper-bound of $\max \{ \gamma_\U^\circ(y),\gamma_\U^\circ(-y)\}^2$  based on $\K$, we also have a convex upper-bound based on $\K^\circ \cap \S_d^+$. Note that it is not true in general that exact representations of one kind transfer to exact representation of the other kind.

  \subsection{Examples}

For the $\ell_2$-norm ball, then we may characterize exactly $\P_\U^\circ$ as $ \P_\U^\circ = \{ M \in \rb^{d \times d }, M \preccurlyeq \idm \}$. This provides an example $\K = \P_\U^\circ \cap \S_d^+ = \{ M \in \S_d^+, \   M \preccurlyeq \idm\}$, for which $\K^\circ \cap \S_d^+ = \{ M \in \S_d^+, \ \tr M \leqslant 1\}$. Another example is 
 $\K = \{ \idm \}$ with $\K^\circ \cap \S_d^+ = \{ M \in \S_d^+, \ \tr M \leqslant 1\}$. Yet another one is
  $\K = \{ M \in \S_d^+, \ \| M \|_F \leqslant 1\}$ with $\K^\circ \cap \S_d^+ = \{ M \in \S_d^+, \ \|M\|_F \leqslant 1\}$.
 
For the $\ell_1$-norm ball, we also have a representation in closed form of the largest set $\K= \P_\U^\circ  \cap \S_d^+ = \big\{
M \in \rb^{d \times d}, \ M \succcurlyeq 0, \ \| \diag(M)\|_\infty \leqslant 1
\big\}$.

Apart from these two sets, a simple (manageable in polynomial time) description of $\P_\U^\circ  \cap \S_d^+$ is not available and smaller sets are generally available. For example, for the $\ell_p$-norm, $p \in [1,\infty]$, the set $\K = \{ M \succcurlyeq 0, \ \|M\|_q \leqslant 1\}$ satisfies \eq{K}. Moreover, if $p \geqslant 2$, $\K =  \{ \Diag(\eta), \ \eta \in \rb_+^d, \| \eta \|_{p/(p-2)} \leqslant 1 \} $ also does~\cite{tal} (because $p/(p-2)=1/(1-2/p)$ and $\sup_{ \eta \geqslant 0, \ 
\| \eta\|_{p/(p-2)} \leqslant 1} \sum_i \eta_i x_i^2 = \| x \circ x\|_{p/2} $, where $\circ$ denotes the pointwise product of vectors). See additional example of diagonal representations in \mysec{diag}.

  \subsection{Diagonal representations}
\label{sec:diag}

We now consider cases where the set of matrices $\K$ and $\L$ are diagonal, i.e., all principal axes of ellipsoids are aligned with the canonical basis. For simplicity, we consider only sets $\U$ which are compact, have zero in their interior, and are invariant by sign flips of any components. The corresponding gauge functions are then \emph{absolute norms}, which are functions of the absolute values of each component~\cite{Stewart1990}. The following proposition provides several characterizations (see related work in~\cite{tal,fot}).

\begin{proposition}
\label{prop:diag}
Let $\Omega$ be an absolute norm on $\rb^d$ and $\Omega^\circ$ its dual (also absolute) norm. The following three conditions are equivalent:
\BNUM
\item[(a)] There exists a non-empty closed convex subset $\H$ of $\rb^d_+$ such that $\forall x \in \rb^d, \Omega(x)^2 = \inf_{ \eta \in \H}
x^\top \Diag(\eta)^{-1} x$.
\item[(b)] The function $t \mapsto \Omega(t^{1/2})^2$ is concave on $\rb_+^d$.
\item[(c)] The function $t \mapsto \Omega^\circ(t^{1/2})^2$ is convex  on $\rb_+^d$.
\ENUM

\end{proposition}
\begin{proof}
It is straightforward to see that (a) implies both (b) and (c).

Assume (b). The Fenchel conjugate $ t \mapsto - \Omega(t^{1/2})^2$ on $\rb_+^d$, defined as
$f(u) = \sup_{ t \in \rb_+^d} t^\top u + \Omega(t^{1/2})^2$ is the indicator function of a convex set $\C$ (because 
$t \mapsto \Omega(t^{1/2})^2$ is positively homogeneous). Moreover $-\C \subset \rb_+^d$, and $u 
\in \C$ if and only if for all $ t \in \rb_+^d$, $\Omega(t^{1/2})^2 \leqslant -u^\top t$, i.e., 
for all $w \in \rb^d$, $\Omega(w)^2 \leqslant w^\top \Diag(-u) w$, which is equivalent to,
for all $s \in \rb^d$, $\Omega^\circ(s)^2 \geqslant s^\top \Diag(-u)^{-1} s$. Thus $-1/\C = \{ - (1/u), \ u \in \C\}$ is also convex.

By Fenchel duality, we get, for all $ t \in \rb_+^d$,
   $ - \Omega(t^{1/2})^2 = \sup_{u \in \C} t^\top u$, i.e., for all $w \in \rb^d$,
$   \Omega(w)^2 = \Omega(|w|)^2 = \inf_{u \in -\C} w^\top \Diag(u) w$, which shows that we can take $\H = -1/\C$ and obtain (a).

Assume (c). Then  the Fenchel conjugate $ t \mapsto   \Omega^\circ(t^{1/2})^2$ on $\rb_+^d$, defined as
$f(u) = \sup_{ t \in \rb_+^d} t^\top u - \Omega(t^{1/2})^2$ is the indicator function of a convex set $\C$ (because 
$t \mapsto \Omega^\circ(t^{1/2})^2$ is positively homogeneous). By duality, we get, for all $ t \in \rb_+^d$,
   $   \Omega^\circ(t^{1/2})^2 = \sup_{u \in \C} t^\top u$, i.e., for all $s \in \rb^d$,
$   \Omega^\circ(s)^2 = \Omega^\circ(|s|)^2 = \sup_{u \in \C} w^\top \Diag(u) w$, which leads to (a) when computing Fenchel conjugates.
\end{proof}

\vspace*{.2500cm}

The construction above gives a way of constructing the largest set $\H \subset \rb_+^d$, such that $\forall w \in \rb^d, \Omega(w)^2 = \inf_{ \eta \in \H}
w^\top \Diag(\eta)^{-1} \eta$:   $\H$ is the set of $\eta \in \rb_+^d$ such that for all $s \in \rb^d$, $   s^\top \Diag(\eta) s  \leqslant \Omega^\circ(s)^2$, then, we have shown above that $1/\H$ is the set of $\zeta \in \rb_+^d$ such that for all $w\in \rb^d$, $   \Omega(w)^2  \leqslant w^\top \Diag(\zeta) w   $, and is also convex.
 
 \paragraph{Examples.}
Among $\ell_p$-norms, only the ones for which $p \leqslant 2$ have diagonal representations, with 
 $\H =  \{ \eta \geqslant 0, \| \eta\|_{p / (2 -p) } \leqslant 1 \}$.
Indeed,  we have for $\Omega = \| \cdot \|_p$ and $1/p+1/q=1$:
$\Omega^\circ(t^{1/2})^2 = \| t \|_{ q/2 } = \sup_{ \| s \|_{1/( 1 - 2 / q) \leqslant 1} } s^\top t$, with
$ 1/( 1 - 2 / q) = 1/ ( 1 - 2( 1 - 1/p) ) = 1/ ( 2/p-1) = p / ( 2 - p)$.

Beyond $\ell_p$-norms, the set $\H$ have also been used to define norms  with specific structured sparsity-inducing properties~\cite{micchelli2013regularizers,fot}
or optimality properties in terms of convex relaxations~\cite{submodlp}. Note moreover, that such representations are useful for optimization as they lead to simple reweighted-$\ell_2$ algorithms for all of these norms~\cite{daubechies2010iteratively,fot}.

\subsection{Matrix norms invariant by rotation}
\label{sec:rotation}

We now consider norms on $\rb^{ n \times d}$ which are invariant by right-multiplication by a rotation, i.e., which are functions of $WW^\top$, $W \in \rb^{n \times d}$. The corresponding unit ball $\U$ is invariant by right-multiplication. The following proposition may be shown using the same arguments than Prop.~\ref{prop:diag}.
\begin{proposition}
\label{prop:rot}
Let $\Omega$ be a norm on $\rb^{n \times d}$ which is invariant by right-multiplication by a rotation, and $\Omega^\circ$ its dual norm. The following three conditions are equivalent:
\BNUM
\item[(a)] There exists a non-empty convex subset $\K$ of $\S_d^+$ such that $\forall W \in \rb^{n \times d }$, $ \Omega(W)^2 = \inf_{ M \in \K }
\tr W^\top M^{-1} W$.
\item[(b)] $W \mapsto \Omega(W)^2$ is a concave function of $WW^\top$.
\item[(c)] $V \mapsto \Omega^\circ(V)^2$ is a concave function of $ V V^\top $.
\ENUM
\end{proposition}

\vspace*{.2500cm}

All norms obtained in \mysec{decomposition} for $\V$ the $\ell_2$-ball will be instances of these norms. 
Moreover, all norms which are invariant by right and left multiplication by a rotation, must be obtained as a spectral function, i.e., $\Omega(W)$ is a symmetric norm on the singular values of $W$.  Quadratic representations based on the spectrum could be derived as well using tools from~\cite{lewis2003mathematics}. Note that for Prop.~\ref{prop:rot}, the set $\K$ is then a spectral set (defined through constraints on eigenvalues).

 \section{Gauge functions and structured matrix factorizations}
 \label{sec:decomposition}
 In this section, we consider two compact (non-necessarily convex) sets $\U \subset \rb^n$ and $\V \subset \rb^d$. We will always assume that 0 is in the convex hulls of both $\U$ and $\V$.

 We consider the decomposition of matrices $X \in \rb^{n \times d}$ as a positive linear combination of matrices of the form $uv^\top$ for $u \in \U$ and $v \in \V$, i.e., $X = \sum_{m=1}^r \lambda_m u_m v_m^\top$, with a priori no constraints on the rank $r$ (it will always be less than $nd$). Finding the decomposition such that $\sum_{m=1}^r \lambda_m$ is minimum leads to a gauge function, with interesting properties.

 \subsection{Definitions}
 Let $ \U \subset \rb^n$ and $\V \subset \rb^d$ be two compact sets. We define the function $\Theta$ on $\rb^{n \times d}$    as follows
 \BEQ
 \label{eq:theta}
 \Theta(X) = \inf_{ r \geqslant 0 } \inf_{ X = \sum_{m=1}^r \lambda_m u_m v_m^\top, \ \lambda_m \in \rb_+, u_m \in \U, v_m \in \V } \sum_{m=1}^r \lambda_m .
 \EEQ
 This function is exactly the gauge function of the set $  \{ uv^\top, \ u \in \U, v\in V\}$. It is therefore convex, non-negative and positively homogeneous. Moreover, the minimization problem defining it is attained for $r \leqslant nd$.  
 There are several equivalent formulations which we are going to use, which relies on the fact that $ u_m v_m^\top = ( u_m \lambda_m) ( v_m\lambda_m^{-1})^\top$ for any $\lambda_m > 0$ (in the following expressions, we always minimize with respect to the rank and thus omit the notation $\inf_{r \geqslant 0}$).
 \begin{proposition}
 We have
 \BEAS
 \Theta(X) & = & \!\!
 \inf_{ X = \sum_{m=1}^r u_m v_m^\top }  \frac{1}{2} \sum_{m=1}^r  \big\{   \gamma_\U(u_m)^2 + \gamma_\V(v_m )^2 \big\} 
 =  \!\! \inf_{ X = \sum_{m=1}^r u_m v_m^\top }   \sum_{m=1}^r    \gamma_\U(u_m) \gamma_\V(v_m )\\[-.1cm]
 & = & \!\!
 \inf_{ X = \sum_{m=1}^r u_m v_m^\top, \ u_m \in \U  }   \sum_{m=1}^r    
\gamma_\V( v_m ) =  \inf_{ X = \sum_{m=1}^r u_m v_m^\top, \ v_m \in \V  }   \sum_{m=1}^r    
 \gamma_\U(u_m).
\EEAS
      \end{proposition}

\vspace*{.2500cm}

Moreover,  $\Theta$ is a norm as soon as $\U$ and $\V$ are norm balls---in this case, these norms were studied in several settings~\cite{jameson1987summing,bach2008convex,chandrasekaran2012convex}.
 The next proposition shows that the polar of $\Theta$ has a simple form (the proof is straightforward from the gauge function interpretation). It is a matrix norm when $\U$ and $\V$ are norm balls; moreover, in all cases, since $\U$ and $\V$ are assumed compact, $\Theta^\circ$ has full domain.
 
 \begin{proposition}
The polar of $\Theta$ is equal to
 $$\Theta^\circ(Y) = \max_{ u \in \U, v \in \V } u^\top Y v
 = \max_{ u \in \U  }   \gamma_\V^\circ(Y^\top u )
 = \max_{ v \in \V } \gamma_\U^\circ( Y v ) 
 = \max_{ \gamma_\U(u) \leqslant 1, \ \gamma_\V(v) \leqslant 1}   u^\top Y v.
 $$
  \end{proposition}
    
\vspace*{.12500cm}

The gauge function $\Theta$ or its polar may not be computed in closed form in general. There are two special cases where this is possible, namely when one of the sets is  the $\ell_1$-norm ball, and when the two sets are the $\ell_2$-norm balls.

\subsection{Examples}
In this paper, we consider several sets $\U$ and $\V$ with interesting properties which are displayed in Table~\ref{tab:U}.  Note that when $\V$ is the $\ell_1$-ball, then the norm $\Theta$ may be computed in closed form as $
\Theta(X) = \sum_{i=1}^d \gamma_\U( X(:,i) )$, where $X(:,i)$ is the $i$-th column of $X$. The case where $\V$ is the $\ell_2$-ball is also specific and will be treated in \mysec{l2}. Some specific combinations of $\U$ and $\V$ are particularly interesting:

\begin{list}{\labelitemi}{\leftmargin=1.1em}
   \addtolength{\itemsep}{-.0\baselineskip}

\item[--] $\U$ and $\V$ are  unit $\ell_2$-balls: the norm $\Theta$ is the nuclear norm (sum of singular values), leading to a regularizer inducing low-rank~\cite{recht2010guaranteed}.

\item[--] $\V$ is the unit $\ell_2$-ball and $\U^\circ$ is the intersection of the unit $\ell_2$-ball and a scaled $\ell_\infty$-ball, i.e.,  $\U^\circ =\{ w \in \rb^n, \   \| w\|_2 \leqslant 1,     \| w\|_\infty \leqslant \nu \}$. The norm may $\Theta$ may then be computed in closed form as 
$\Theta(X) = \inf_{X = Y + Z } \| Y\|_\ast + \nu \sum_{i=1}^n \|Z(i,:)\|_2$, and has been used in robust versions of principal component analysis where some observations may be corrupted (through a ``low-rank + group-sparse'' model)~\cite{xu2012robust}.

\item[--] $\gamma_\V = \| \cdot \|_2$ and $\gamma_\U = \sqrt{ \nu \| \cdot \|_2^2 + (1-\nu) \| \cdot \|_1^2 }$: this is a convex relaxation of sparse coding~\cite{bach2008convex}, where a decomposition with a sparse factor $U$ and low rank (small number of columns for $U$ and $V$) is looked for.

\item[--] $\U$ and $\V$ are $\ell_\infty$-unit balls: the norm $\Theta$ has been considered as a complexity measure for sign matrices~\cite{linial2007complexity}. It is moreover related to the cut-norm~\cite{alon2006approximating}; the relaxation presented in \mysec{relaxation} is the max-norm~\cite{srebro2005rank,lee2010practical}. Note here that our conditional gradient algorithms presented in~\mysec{algorithms} significantly improves the number of rank-one factor to obtain a $\varepsilon$-approximate decomposition from $O(1/\varepsilon^2)$~\cite{alon2006approximating} to $O (\log \frac{ 1 }{\varepsilon} )$.

\item[--] $\gamma_\V = \| \cdot \|_2$ and $\gamma_\U = \| \cdot  \|_{4/3}$: we have $\gamma_\U^\circ = \| \cdot \|_4$, and the dual norm is such that  $\Theta^\circ(Y)^4 = \max_{ \|v\|_2 \leqslant 1}  \sum_{i} ( Yv)_i^4$, and corresponds to maximum kurtosis projections~\cite{hyvarinen1999fast}.

\item[--] $\U$ and $\V$ are $\ell_p$-norm balls. This is a case considered by~\cite{tal}, where approximation guarantees are derived. Values of $p$ for which dimension-independent guarantees are obtained are exactly the ones corresponding to diagonal representations in Prop.~\ref{prop:diag2}.

\item[--] $\U$ and $\V$ are $\ell_p$-norm balls intersected with the positive orthant. We obtain convex relaxations of non-negative matrix factorization. Note that for $p=\infty$, it comes with approximation guarantees in terms of the value of the gauge function $\Theta$.

\item[--] $\U = \{0,1\}^n$ and $\gamma_\V = \| \cdot \|_2$: this corresponds to decomposing the matrix $X$ with a factor $U$ with binary values. 
We will consider this example in our experiments in \mysec{experiments}.
Extensions may be considered by (a) adding also a constraint that less than $k$ components of $u$ are non-zeros (i.e., by adding a constraint $u^\top 1_n \leqslant k$), or (b) considering $\U = \{ \frac{1_A}{F(A)}, \ A \subset V, \ A \neq \varnothing \}$, for a non-decreading submodular function $F$, leading to $\gamma_\U(x) = f(x) + I_{\rb_+^n}(x)$ where $f$ is the Lov\`asz extension of~$F$~\cite{fujishige2005submodular,bach2011learning}.

\end{list}

\begin{table}
\caption{Examples of structured sets $\U$, together with a computable surrogate $\C_\U$ of $\P_\U = \{ uu^\top, \ u \in \U\}$, that will be used in \mysec{relaxation}. Non-negativity constraints $u \geqslant 0$ may be added to $\U$ with an additional constraint $M \geqslant 0$ for $\C_\U$.}
\label{tab:U}

\begin{center}
\begin{tabular}{|l|l|l|}
\hline
& $\U$ & $\C_\U$ \\
\hline
$\ell_p$-ball, $p \in [1,2]$ & $\|u\|_p \leqslant 1$ & $M \succcurlyeq 0, \| M \|_p \leqslant 1, \tr M \leqslant 1$ \\
$\ell_p$-ball, $p \in [2,+\infty]\!\!$ & $\|u\|_p \leqslant 1$ & $M \succcurlyeq 0, \| \diag(M) \|_{p/2} \leqslant 1$  \\
sparse coding & $\nu \| u\|_2^2 + (1 \!-\! \nu) \|u\|_1^2 \leqslant 1 \!\! $ & $M \succcurlyeq 0, \nu \tr M + (1 \!-\! \nu) \|M\|_1 \leqslant 1$ \\
binary coding & $u \in \{0,1\}^d$ &  $M \succcurlyeq mm ^\top, M  \geqslant 0, m = \diag(M) \leqslant 1 \!\!$  \\
\hline
\end{tabular}
\end{center}
\end{table}

  \subsection{Special case: $\gamma_\V = \| \cdot \|_2$}
  \label{sec:l2}
  Computing the polar of $\Theta$ then corresponds to a quadratic maximization problem:
  \BEQ
  \label{eq:PU}
  \Theta^\circ(Y)^2  = \max_{ u \in \U }   { u^\top YY^\top u }= \max_{ u \in \U^{\circ\circ} }   { u^\top YY^\top u }=  \max_{ \gamma_\U(u) \leqslant 1 }   { u^\top YY^\top u }.
 \EEQ

We also have a variational quadratic representation corresponding to a rotation-invariant gauge function (\mysec{rotation}); following \mysec{K}, we denote by $\P_\U$ the set $\{ uu^\top, \ u \in \U\}$. There is an explicit representation of $\Theta$ and $\Theta^\circ$ in terms of   $\P_\U^{\circ\circ}$ (the convex hull of $\P_\U \cup \{0\}$).

 \begin{proposition}
 We have, for all $X,Y \in \rb^{n \times d}$: $$ \Theta(X)^2  = \inf_{  M \in \P_\U^{\circ\circ} }      \tr X^\top M^{-1} X 
\  \mbox{ and } \  \Theta^\circ(Y)^2 = \max_{ M \in \P_\U^{\circ\circ}} \tr M YY^\top. $$
 \end{proposition}
 \begin{proof}
 The representation of $\Theta^\circ$ is obvious from \eq{PU}. We then have, for any $X \in \rb^{n \times d}$:
  \BEAS
  \Theta(X)^2 &  = &  \sup_{ Y \in \rb^{n \times d} } \tr X^\top Y - 2 \Theta^\circ(Y)^2 \mbox{ like all polar pairs}, \\
   & = & \sup_{ Y \in \rb^{n \times d} } \tr X^\top Y - 2  \max_{ M \in \P_\U^{\circ\circ} }  \tr  YY^\top M \mbox{ from \eq{PU}},\\
  & = & \inf_{ M \in \P_\U^{\circ\circ} }  \sup_{ Y \in \rb^{n \times d} } \tr X^\top Y - 2   \tr  YY^\top M \mbox{ by Fenchel duality}, \\
  & = & \inf_{ M \in \P_\U^{\circ\circ} }  \tr X^\top M^{-1} X.
  \EEAS
  Note that $\Theta$ may not have full domain.
  \end{proof}

\vspace*{.2500cm}

  From the earlier representation, we get $  \displaystyle \Theta(X)  = \inf_{ M \succcurlyeq 0 }  \big\{ \frac{1}{2} \gamma_{\P_\U}(M)  + \frac{1}{2} \tr X^\top M^{-1} X \big\}$, by simply optimizing over the half line generated by $M$. We also have a representation involving all decompositions of $X$ as $UV^\top$:
 \BEQ
 \label{eq:UV2} \Theta(X)  = \min_{ X = UV^\top }  
 \frac{1}{2} \tr VV^\top + \frac{1}{2} \gamma_{\P_\U} ( UU^\top),
 \EEQ
 which is straightforward from the definitions and interpretations of $\Theta$ and  $\gamma_{\P_\U}$ as a gauge function (Prop.~\ref{prop:UU}). In the next section, we show how to compute $U$ and $V$ from the solution of a certain convex problem---note however that this does not lead to the optimal $U$ and $V$ in \eq{theta}. See more details in \mysec{algorithms}.

\subsection{General case}
We now consider all possible cases, beyond $\gamma_\V = \| \cdot \|_2$ (where $\gamma_{\P_\V}(M) = \tr M $). We only assume that $\U$ and $\V$ are compact and contain zero in their hulls.
  It is tempting to consider an extension of \eq{UV2}, by considering
\BEQ  
\label{eq:tilde} 
\widetilde{\Theta} (X) =  \! \min_{ X = UV^\top }   \!
 \frac{1}{2}\big[ \gamma_{\P_\V}(VV^\top) +  \gamma_{\P_\U} ( UU^\top) \big]
 =  \! \min_{ X = UV^\top }   \! \!
\sqrt{\gamma_{\P_\V}(VV^\top)  \gamma_{\P_\U} ( UU^\top)} .  \!\EEQ
 However, when neither $\U$ nor $\V$ are unit $\ell_2$-balls, this only provides a lower-bound to $\Theta$, because the optimal decompositions of $UU^\top$ and $VV^\top$ for respectively $\gamma_{\P_\U}$ and $\gamma_{\P_\V}$ may not lead to an optimal decomposition of $X$. Since all our further relaxations in \mysec{relaxation} will work directly on $\widetilde{\Theta}$, we study precisely the properties of $\widetilde{\Theta}$. 
 We first show that \eq{tilde} defines a gauge function and give a novel expression of its polar, as a maximum of weighted nuclear norms (note that the function
 $ (M,N) \mapsto \| M^{1/2} Y N^{1/2} \|_\ast$ is concave on $\S_n^+ \times \S_d^+$).
 
 \begin{proposition}
 \label{prop:tilde0}
 Let $\widetilde{\Theta}$ be defined in \eq{tilde}. Then, $\widetilde{\Theta}$ is a gauge function (convex, positively homogeneous, and non-negative), and its polar has the expression:
 \BEQ
 \forall Y \in \rb^{n \times d}, \ \widetilde{\Theta}^\circ(Y) =  \sup_{ M \in \P_\U^{\circ\circ},   \ N  \in \P_\V^{\circ\circ} } \| M^{1/2} Y N^{1/2} \|_\ast.
 \EEQ
 Moreover, for all $X, Y \in \rb^{n \times d}$,
$\Theta(X) \geqslant \widetilde{\Theta}(X)$ and  $\Theta^\circ(Y) \leqslant \widetilde{\Theta}^\circ(Y)$.
 \end{proposition}
  \begin{proof}
  The three properties of $\widetilde{\Theta}$ are straightforward. We  have, for $Y \in \rb^{n \times d}$:
  \BEAS
\!\!\!  \widetilde{\Theta}^\circ(Y) \!& = &\!\!\! \sup_{  \gamma_{\P_\V}(VV^\top)  \gamma_{\P_\U} ( UU^\top) \leqslant 1 } \tr Y^\top U V^\top
= \sup_{  \gamma_{\P_\V}(VV^\top) \leqslant 1, \   \gamma_{\P_\U} ( UU^\top) \leqslant 1 } \tr Y^\top U V^\top \\
  & = &\!\!\! \sup_{ M \in \P_\U^{\circ\circ},   \ N \in \P_\V^{\circ\circ} } \   \sup_{  UU^\top =M, \ VV^\top = N} \tr Y^\top U V^\top \\
  & = &\!\!\! \sup_{ M \in \P_\U^{\circ\circ},   \ N \in \P_\V^{\circ\circ} } \   \sup_{  UU^\top = \idm, \ VV^\top = \idm} \tr Y^\top M^{1/2} U V^\top N^{1/2} =  \!\!\! \sup_{ M \in \P_\U^{\circ\circ},   \ N  \in \P_\V^{\circ\circ} } \| M^{1/2} Y N^{1/2} \|_\ast,
  \EEAS
  because $\sup_{ UU^\top = \idm, \ VV^\top = \idm} \tr Z^\top UV^\top = \| Z\|_\ast$.
  Finally, we have for some $u_m,v_m$, $\Theta(X)  = \frac{1}{2} \sum_{m} \big\{ \gamma_\U(u_m)^2 + \gamma_\V(v_m)^2 \big\} $. Thus $\Theta(X) \geqslant \widetilde{\Theta}(X)$ because of Prop.~\ref{prop:UU}, which implies the reverse inequality for polars.
  \end{proof}
  
\vspace*{.2500cm}

 We may now give several equivalent expressions for $\widetilde{\Theta}$ and $\widetilde{\Theta}^\circ$ which are obtained using usual convex duality.
 
\begin{proposition}
\label{prop:tilde}
 Let $\widetilde{\Theta}$ be defined in \eq{tilde}. We have:
\BEA
\nonumber \widetilde{\Theta}(X) & = &   \inf_{M \in \S_n, \ N \in \S_d }   \frac{1}{2}\gamma_{ \P_\U} (M) + \frac{1}{2} \gamma_{ \P_\V} (N)  \mbox{ s.t. } 
  \Big( \! \begin{array}{cc} M & X \\ X^\top & N \end{array} \! \Big) \succcurlyeq 0 \\
\nonumber  & = &  \sup_{ Q \in \P_\U^{\circ}, \ S \in \P_\V^{\circ}, \ Z \in \rb^{n \times d}}
\frac{1}{2}
\tr 
\Big( \!\begin{array}{cc} Q  & Z \\ Z^\top & S \end{array} \!\Big) 
\Big( \!\begin{array}{cc} 0  & X \\ X^\top  & 0  \end{array} \!\Big) 
\mbox{ s.t. } \bigg( \begin{array}{cc} Q  & Z \\ Z^\top & S \end{array} \bigg) \succcurlyeq 0,
\\
\nonumber \widetilde{\Theta}^\circ(Y) & = &   \inf_{Q \in \S_n, \ S \in \S_d}   \frac{1}{2}\gamma_{ \P_\U^\circ} (Q) + \frac{1}{2} \gamma_{ \P_\V^\circ} (S)  \mbox{ s.t. } 
  \Big( \!\begin{array}{cc} Q & Y \\ Y^\top & S \end{array} \!\Big) \succcurlyeq 0 \\
\nonumber   & = &  \sup_{ M \in \P_\U^{\circ\circ}, \ N \in \P_\V^{\circ\circ}, \ Z \in \rb^{n \times d}}
\frac{1}{2}
\tr 
\Big( \!\begin{array}{cc} M  & Z \\ Z^\top & N \end{array} \!\Big) 
\Big( \!\begin{array}{cc} 0  & Y \\ Y^\top  & 0  \end{array} \!\Big) 
\mbox{ s.t. } \Big( \! \begin{array}{cc} M  & Z \\ Z^\top & N  \end{array} \!\Big) \succcurlyeq 0.
  \EEA
\end{proposition}
\begin{proof}
The four expressions are equivalent by taking polars or using Lagrangian duality for the semi-definite cone. The only element left to prove is that any of these four statements are correct.
We have: 
$$ \max_{   Z \in \rb^{n \times d} }
\tr Z^\top Y 
\mbox{ s.t. } \Big( \! \begin{array}{cc} M  & Z \\ Z^\top & N  \end{array} \!\Big) \succcurlyeq 0 =  \max_{   \| T\|_{\rm op} \leqslant 1 } \tr Y^\top M^{1/2}  T N  ^{1/2}
=  \| M^{1/2} Y N^{1/2} \|_\ast,
$$
which leads to desired result from Prop.~\ref{prop:tilde0}.
 Note that we may obtain directly a relationship between $\widetilde{\Theta}^\circ$ and $\Theta^\circ$ as follows:
\BEAS
\! \! \Theta^\circ(Y) & = &  \! \max_{u \in \U, \ v \in \V } v^\top Y^\top u 
=  \max_{u \in \U, \ v \in \V }  \frac{1}{2}
\tr 
\Big( \!\begin{array}{cc} uu^\top  & uv^\top \\ vu^\top & uu^\top  \end{array} \!\Big) 
\Big( \!\begin{array}{cc} 0  & Y \\ Y^\top  & 0  \end{array}\! \Big) 
\\
& = & \!  \max_{ M \in \P_\U, \ N \in \P_\V}
\frac{1}{2}
\tr 
\Big( \!\begin{array}{cc} M  & Z \\ Z^\top & N  \end{array} \!\Big) 
\Big(\! \begin{array}{cc} 0  & Y \\ Y^\top  & 0  \end{array} \!\Big) 
\mbox{ s.t. } \Big( \!\begin{array}{cc} M  & Z \\ Z^\top &  N \end{array}\! \Big) \succcurlyeq 0
\\
& \leqslant &\!   \max_{ M \in \P_\U^{\circ\circ}, \ N \in \P_\V^{\circ\circ}}
\frac{1}{2}
\tr 
\Big( \!\begin{array}{cc} M  & Z \\ Z^\top & N \end{array} \!\Big) 
\Big( \!\begin{array}{cc} 0  & Y \\ Y^\top  & 0  \end{array} \!\Big) 
\mbox{ s.t.  } \Big( \!\begin{array}{cc} M  & Z \\ Z^\top & N  \end{array}\! \Big) \succcurlyeq 0 = \widetilde{\Theta}^\circ(Y).
   \EEAS
 \end{proof}

\vspace*{.2500cm}

Note that the representations of $\widetilde{\Theta}$ and $\widetilde{\Theta}^\circ$ have the same form, replacing $\P_\U$ by $\P_\U^{\circ} \cap \S_d^+$. We thus get that for all $X \in \rb^{n \times d}$:
\BEQ
 {\Theta}(X) \geqslant \widetilde{\Theta}(X) =  \sup_{ Q \in \P_\U^{\circ} \cap \S_n^+,   \ S \in \P_\V^{\circ} \cap \S_d^+} \| Q^{1/2} Y S^{1/2} \|_\ast,
\EEQ
that is a variational formulation of $\widetilde{\Theta}$ through nuclear norms (note again that $\P_\U^\circ$ and $\P_\V^\circ$ may not be compact and $\widetilde{\Theta}(X)$ potentially infinite). This is to be contrasted with the following representation:
$$
\!\! 
\Theta^\circ(Y)^2 = \!\! \sup_{ M \in \P_\U, \ N \in \P_\V} \!\!  \tr M Y N Y^\top = \!\!  \sup_{ M \in \P_\U^{\circ\circ}, \ N \in \P_\V^{\circ\circ}} \!\!  \tr M Y N Y^\top =  \!\!  \sup_{ M \in \P_\U, \ N \in \P_\V} \!\!  \| M^{1/2} Y N^{1/2} \|_F^2 ,
$$
which leads to
\BEAS
\Theta(X)^2 & = & \sup_{ Y \in \rb^{n \times d}} \inf_{ M \in \P_\U^{\circ\circ}, \ N \in \P_\V^{\circ\circ}} \tr X^\top Y  - 2\tr M Y N Y^\top \\
& \leqslant & \inf_{ M \in \P_\U^{\circ\circ}, \ N \in \P_\V^{\circ\circ}} \sup_{ Y \in \rb^{n \times d}}  \tr X^\top Y  - 2 \tr M Y N Y^\top 
\mbox{ by weak duality,}\\
& = & \inf_{ M \in \P_\U^{\circ\circ}, \ N \in \P_\V^{\circ\circ}}  \tr M^{-1} X N^{-1} Y^\top, \\
\Theta(X) & \leqslant & \inf_{ M \in \P_\U^{\circ\circ}, \ N \in \P_\V^{\circ\circ}}  \sqrt{ \tr M^{-1} X N^{-1} Y^\top}
= \inf_{ M \in \P_\U^{\circ\circ}, \ N \in \P_\V^{\circ\circ}}  \| M^{-1/2} X N^{-1/2} \|_F,
\EEAS
which is a variational upper-bound as a weighted Frobenius norm, which is not computable as a convex program. Instead, alternate minimization could be used, each of the steps being doable as a convex program.

\paragraph{Relationships between $Q$, $S$, $M$, $N$, $U$ and $V$.} 
\label{sec:QAB}
Given the primal/dual solutions $(M,N)$, $(Q,S)$ of the dual convex optimization problems defining $\widetilde{\Theta}$ in Prop.~\ref{prop:tilde}, we have the following optimality conditions:
Assume $Q = AA^\top$ with $A$ full rank (i.e., $A^\top A$ invertible), and $S=BB^\top$ with $B$ full rank. Then an optimal $Z$ is  $Z = A G H^\top B^\top$, with $A^\top X B = G \Diag(s) H^\top$ is the singular value decomposition of $A^\top X B$. We have 
$$ \Big( \! \begin{array}{cc} M & -X \\ -X^\top & N \end{array} \! \Big) \succcurlyeq 0 \mbox{ and }   \Big( \! \begin{array}{cc} M & -X \\ -X^\top & N \end{array} \! \Big)  \Big(\!  \begin{array}{cc} AA^\top & AGH^\top B^\top \\ B HG^\top A ^\top & BB^\top \end{array} \! \Big) = 0,$$
 which implies (since $A$ and $B$ have full rank):
$$  \Big( \! \begin{array}{cc} M & -X \\ -X^\top & N \end{array} \! \Big)  \Big( \! \begin{array}{cc} A & AGH^\top  \\ B VU^\top   & B  \end{array} \! \Big) = 0.$$
   This leads to $A^\top MA = A^\top X B HG^\top = G \Diag(s) G^\top$. We then take as candidates $U =  (AA^\top)^{\dagger}  A G \Diag(s)^{1/2}$ and
   $V =  (BB^\top)^{\dagger} B H \Diag(s)^{1/2}$, leading to
   $
   UU^\top =M
   $ and $VV^\top = N$. Moreover, $X = UV^\top$ and
   $
\widetilde{\Theta}(X) =      \| A^\top X B\|_\ast = \tr \Diag(s) $. Finally, by optimality of $Q$, we have
$\gamma_{\P_\U}(UU^\top) = \tr UU^\top Q = \tr A^\top M A = \tr \Diag(s)$ and 
$\gamma_{\P_\V}(VV^\top) = \tr VV^\top S = \tr B^\top N B = \tr \Diag(s)$, and thus $(U,V)$ is an optimal decomposition for
\eq{tilde}.
  
  \paragraph{Validity for other sets than $\P_\U$ and $\P_\V$.} The duality properties presented in Prop.~\ref{prop:tilde} are valid for all convex sets of positive semi-definite matrices that contain zero (and in particular for the ones in \mysec{relaxation}). This will be used in \mysec{relaxation} to get computable approximations (with guarantees).

    \subsection{Non-convex approaches}
    \label{sec:nonconvex}
    
Finding approximations of the gauge function~$\Theta$ or its polar $\Theta^\circ$ has been tackled  thoroughly through non-convex approaches, which may only find stationary points of the associated optimization problems. In this paper, we briefly review two approaches based on alternating minimization.

   \paragraph{Computing polar function through a generalized power method.}
 The task is to find  maximizers $(u,v) \in \U \times \V$ of $u^\top Y v$, for a certain $Y \in \rb^{n \times d}$.      The generalized power method iterates the recursion $ u \in \arg\max_{u \in \U} u^\top Yv$ and $ v \in \arg\max_{v \in \V} u^\top Yv$ (i.e., alternate maximization), started from a certain $u$ (typically random, or obtained from  the relaxation   in \mysec{relaxation}). See an analysis in the context of sparse principal component analysis in~\cite{journee2010generalized}.
   
Each iteration increases $u^\top Y v$, and hence the values are converging. However, the iterates themselves do not converge in general, and when they do, may not converge to a global maximizer. There are certain situations where they do: when $\U$ and $\V$ are $\ell_2$-balls (in this case, this is the usual power method~\cite{golub}), or when $\U$ and $\V$ are $\ell_p$-balls and $Y$ has only non-negative elements~\cite{boyd1974power}.
   
   In \mysec{algorithms}, obtaining approximate maximizers $(u,v)$ will be key for the generalized conditional gradient algorithm and simplicial methods.   In our experiments in \mysec{simulations}, the power method leads to better values once initialized from results of the convex relaxations presented in \mysec{relaxation}.

   \paragraph{Alternating optimization.}
   We consider $r \leqslant np$ elements and minimize either the sum 
  $
  \sum_{m=1}^r \gamma_\U(u_m) \gamma_\V(v_m) 
   $ or
   $
  \sum_{m=1}^r \big\{ \frac{1}{2}\gamma_\U(u_m)^2 + \frac{1}{2} \gamma_\V(v_m)^2 \big\}
   $, subject to $X = \sum_{m=1}^r u_m v_m^\top$. This may be done by alternating minimization, with no guarantees, beyond decreasing the value of the upper-bound on the norm.
   
  Note that in many applications, what needs to be solved is an approximation problem of the form $\min_{Z \in \rb^{n \times d}} \frac{1}{2} \| X - Z\|_F^2 + \lambda \Theta(Z)$, which may be tackled with similar algorithms. In fact, in our simulations in \mysec{simulations}, we solve this problem as it  leads to empirically easier optimization problems and more stable results.
  
\paragraph{Non-convex optimization without local minima.}
As all optimization problems on positive semi-definite matrices, computing the lower-bound $\widetilde{\Theta}(X)$ through Prop.~\ref{prop:tilde} may be done using a low-rank representation of the matrix $\Big( \!\begin{array}{cc} M \! & X \\ X^\top 
\! & N \end{array} \!\Big) = 
\Big( \! \begin{array}{c} U \\ V \end{array} \!\Big) 
\Big( \!\begin{array}{c} U \\ V \end{array}\! \Big)^\top
$. As shown by~\cite{burer2003nonlinear}, if the number of columns of $U$ is greater than the rank of any solution of the semi-definite program, then the non-convex optimization problem has no local minima. However, this result should still be taken with caution in our context: (a) the non-convex problem has no local minima but still have many stationary points and care should still be taken when using iterative algorithms (see, e.g.,~\cite{journee2010low}, for approaches based on trust-region  methods and manifold-based optimization); moreover, (b) the result only applies to $\widetilde{\Theta}(X)$ and not to $ {\Theta}(X)$, (c) even in the situations where $\widetilde{\Theta}(X) = \Theta(X)$ (e.g., when of the two sets is an  unit $\ell_2$-ball), such an approach requires to be able to compute $\gamma_{\P_\U}$, which is not possible in many cases (the max-norm considered in~\cite{lee2010practical} corresponds to the convex relaxation $\gamma_{\C_\U}$ of $\gamma_{\P_\U}$ presented in \mysec{relaxation}); (d) such local-search non-convex techniques cannot be used for finding a decomposition of $M$ as $UU^\top$ where all gauge functions $\gamma_\U$ of columns of $U$ are small (which can however be done using the convex optimization techniques presented in \mysec{algorithms}).

Nevertheless, in practice, alternating optimization techniques, although they come with no guarantees, still perform well when they can be applied (see \mysec{simulations}).

\section{Computable convex relaxations of decomposition norms}
\label{sec:relaxation}

In the previous section, we have given several representations of the gauge function $\Theta$ and its polar~$\Theta^\circ$. They rely on the convex hull $\P_\U^{\circ\circ}$ of the union of $\{0\}$ and
 $\P_\U = \{ uu^\top, \ u \in \U\}$.  Thus the algorithms may have polynomial time if this convex hull may be described in polynomial time. This is notably the case where $\U$ is the unit $\ell_2$-ball, in which case  $\P_\U^{\circ\circ} = \{ M \in \S_n^+, \ \tr M \leqslant 1\}$. We are not aware of any other interesting case for which we have such semi-definite representations.
 
Before describing computable relaxations, we first refine the computability requirement which is weaker than having a good representation of $\P_\U^{\circ\circ}$ or its polar $\P_\U^\circ$. We start with a simple lemma:
\begin{lemma}
\label{lemma:P}
Let $\mathcal{A},\mathcal{B} \subset \S_n^+$. The following statements are equivalent:

(a) $\forall M \in \S_n^+$, $ \sup_{ N \in \mathcal{A} } \tr MN \leqslant \sup_{ N \in \mathcal{B} } \tr MN$,

(b) $\mathcal{A}^{\circ \circ} - \S_n^+ \subset \mathcal{B}^{\circ \circ} - \S_n^+$,

(c) $\mathcal{A}^\circ \cap \S_n^+ \supset \mathcal{B}^\circ \cap \S_n^+$.

Moreover, if  $\mathcal{A}$ is composed of rank-one matrices, then they are also equivalent to

(d) $\forall x \in \rb^n$, $ \sup_{ N \in \mathcal{A} } x^\top N x \leqslant \sup_{ N \in \mathcal{B} } \tr x^\top N x$.

\end{lemma}
 \begin{proof}
 (a) is equivalent to (c) by simple properties of gauge functions. (b) is equivalent to (c) because of polar calculus and $(\S_n^+)^\circ = - \S_n^+$, which imply
 $(\mathcal{A} ^\circ \cap \S_n^+)^\circ = {\rm hull}( \mathcal{A} ^{\circ \circ} \cup ( - \S_n^+) ) = 
  \mathcal{A} ^{\circ \circ}   - \S_n^+$.
  (a) trivially implies (d). We only need to show (d) implies (a). Assume (d) is true and let $M = YY^\top \in \S_n^+$, we have
 $$
 \sup_{ N \in \mathcal{A} } \tr MN
 = \sup_{ uu^\top \in \mathcal{A} } \ \sup_{\|v\|_2 \leqslant 1} (u^\top Y v )^2
\leqslant  \sup_{\|v\|_2 \leqslant 1} \ \sup_{ M \in \mathcal{B} }   \tr Y vv^\top Y^\top M
\leqslant  \sup_{ M \in \mathcal{B} }   \tr  Y^\top M Y, $$
hence the result.
 \end{proof}

\vspace*{.2500cm}

Given the representation of $\widetilde{\Theta}$ in Prop.~\ref{prop:tilde},
 a key consequence of the previous lemma is that in the representation of $\widetilde{\Theta}(X)$
 and    $\widetilde{\Theta}^\circ(Y)$  in \mysec{decomposition}, we may replace $\P_{\U}^{\circ \circ}$ by any set $\C_\U$ such that 
 $$
 \forall M \in \S_n^+, \sup_{u \in \U} u^\top M u  = \sup_{ N \in \mathcal{C}_\U } \tr MN.
 $$
 This allows us to treat the case where $\U$ is the unit $\ell_1$-ball, for which
 $\forall M \in \S_n^+$, $\max_{u \in \U} u^\top M u = \| \diag(M) \|_\infty = \| M \|_\infty$. We may thus take
$\C_\U =   \{ M \in \S_n^+, \ \| M\|_1 \leqslant 1\}$. Note however that it is  not true that $\P_\U^{\circ\circ} = \{ M \in \S_n^+, \ \| M\|_1 \leqslant 1\}$ (counter-examples may be found in~\cite{pcacounter}).

For all other cases, we need computable  convex approximations $\C_\U$ of $\P_\U^{\circ\circ}$ such that for
 $$
 \forall M \in \S_n^+, \sup_{u \in \U} u^\top M u  \leqslant \sup_{ N \in \mathcal{C}_\U } \tr MN
 $$
 (if we impose the property above for all symmetric matrices $M$, it is equivalent to $\P_\U^{\circ\circ} \subset \C_\U$).
 Given Lemma~\ref{lemma:P}, this is equivalent to
 \BEQ
 \label{eq:CU}
 \forall y \in \rb^n, \ \max \{ \gamma_\U^\circ(y),\gamma_\U^\circ(-y)\}^2 = \sup_{ u\in \U} (u^\top y)^2 \leqslant 
  \sup_{ N \in \mathcal{C}_\U } y^\top N y = \gamma_{\C_\U}^\circ(yy^\top),
 \EEQ
 that is, we need an \emph{upper-bound} on $\max \{ \gamma_\U^\circ(y),\gamma_\U^\circ(-y)\}^2$, which is convex in $yy^\top$. Note that given two sets $\C_\U$ that satisfy \eq{CU}, their intersection also does and the relaxation is then always tighter. We can therefore use several properties of the set~$\U$, that may then be combined together:

\begin{list}{\labelitemi}{\leftmargin=1.1em}
   \addtolength{\itemsep}{-.0\baselineskip}

\item[--] \emph{Representation of $\gamma_\U$ as minima of quadratic functions} (\mysec{L}): 
this corresponds to the case where $\gamma_\U^\circ$ is represented as a maximum of quadratic functions, i.e., \eq{CU} is an equality. In this situation, the set $\C_\U$ is used directly. For example, when $\U$ it the unit $\ell_p$-ball, we may choose $\C_\U = \{ M \in \S_n^+, \ \| M \|_p \leqslant 1\}$.

\item[--] \emph{Representation of $\gamma_\U$ as maxima of quadratic functions} (\mysec{K}): as shown in \mysec{L}, if
$\gamma_{\U \cap (-\U)}(x)^2 = \sup_{M \in \K_\U} x^\top M x$, then \eq{CU} is satisfied for $\C_\U = \K_\U^\circ \cap \S_n^+$.
Note that reduced representations may also be useful. For example, when $\gamma_\U^\circ$ has a diagonal representation outlined in \mysec{diag}, we may choose $\K_\U$ to be of the form $\{ \Diag(\eta), \ \eta \in \H_\U \}$, then $\K_\U^\circ \cap \S_n^+ = \{ M \in \S_n^+, \diag(M) \in \H_\U^\circ\}$ may be considered as a set $\C_\U$. For example, when $\U$ is the $\ell_p$-norm unit ball, for $p\geqslant 2$, we obtain $\C_\U = \{ M \in \S_n^+, \ \| \diag(M)\|_{p/2} \leqslant 1 \}$.

\item[--] \emph{Positivity}: if $\U \subset \rb_+^n$, then $\P_\U^{\circ \circ} \subset \{ M \in \S_n^+, \ M \geqslant 0\}
$.

\item[--] \emph{Linear inequalities}: any matrix $M$ such that $\max_{ \gamma_\U(u) \leqslant 1 } u^\top M u$ may be computed or efficiently upper-bounded by $h$ adds another constraint of the form $\tr M  U \leqslant h$. Typically, $M = \idm$ is a good candidate, as this defines the equivalence between the gauge function $\gamma_\U$ and $\| \cdot \|_2$.

\end{list}

In Table~\ref{tab:U}, we present relaxations for the examples we consider in this paper.
For all relaxations we consider (and also with positivity constraints), we have: $\C_\U \cap \{ M, \ {\rm rank}(M)=1\} = \P_\U$. This implies that for all
$u \in {\rm dom}(\gamma_\U)$, $\gamma_{\P_\U}(uu^\top) = \gamma_{\C_\U}(uu^\top) = \min \{ \gamma_\U(u),\gamma_\U(-u) \}^2$.
Moreover,  for all relaxations in  Table~\ref{tab:U} (without positivity constraints), we have equality in \eq{CU}, a property that will be useful when deriving performance guarantees in the next sections~(Prop.~\ref{prop:nd}).

\subsection{Performance guarantees}
\label{sec:guarantees}
Our computable relaxations, will be based on replacing the convex sets $\P_\U^{\circ \circ}$ and $\P_\V^{\circ \circ}$ by  surrogates $\C_\U \subset \S_n^+$ and $\C_\V \subset \S_d^+$ such that $\C_\U - \S_n^+ \subset \P_\U - \S_n^+$ and $\C_\V - \S_n^+ \subset \P_\V - \S_n^+$ (see Lemma~\ref{lemma:P}). This leads to the relaxations, for $X,Y \in \rb^{n \times d}$:
\BEA
 \Theta (X) & \geqslant &  \inf_{M \in \S_n^+, \ N \in \S_d^+}   \frac{1}{2}\gamma_{ \C_\U} (M) + \frac{1}{2} \gamma_{ \C_\V} (N)  \mbox{ s.t. } 
  \Big(\! \begin{array}{cc} M & X \\ X^\top & N \end{array}\! \Big) \succcurlyeq 0 \\
  & = & \sup_{ Q \in \C_\U^\circ \cap \S_n^+, \ S \in \C_\V^\circ \cap \S_d^+ } \| Q^{1/2} X S^{1/2}\|_\ast, \\
\label{eq:relaxTT}
\Theta^\circ (Y) & \leqslant &    \sup_{ M \in \C_\U , \ N \in \C_\V  } \| M^{1/2} Y N^{1/2}\|_\ast,
\EEA
which specializes to  the following expressions when $\V =\{ v \in \rb^d, \  \| v\|_2 \leqslant 1\}$:
\BEA
\Theta (X) & \geqslant &  \inf_{M \in \S_n^+ }   \ \frac{1}{2}\gamma_{ \C_\U} (M) + \frac{1}{2} \tr X^\top M^{-1} X \\
  & = & \sup_{ Q \in \C_\U^\circ  \cap \S_n^+} \| Q^{1/2} X \|_\ast, \\
\label{eq:relax2}
\Theta^\circ (Y) & \leqslant &    \sup_{ M \in \C_\U  } \sqrt{\tr  Y^\top M Y}.
\EEA
Because the bounds are also polar to each other (from Prop.~\ref{prop:tilde}), if we have a performance guarantee for $\Theta^\circ$, i.e.,
\BEQ
\label{eq:rhotilde}
 \forall Y \in \rb^{n \times d}, \ \Theta^\circ (Y)  \leqslant \sup_{ M \in \C_\U , \ N \in \C_\V  } \| M^{1/2} Y N^{1/2}\|_\ast \leqslant \kappa  \Theta^\circ (Y),
 \EEQ
for some $\kappa \geqslant 1$, 
we immediately get
\BEQ
\label{eq:rho}
 \forall X \in \rb^{n \times d}, \ \Theta (X)  \geqslant \sup_{ Q \in \C_\U^\circ \cap \S_n^+ , \ S \in \C_\V^\circ \cap \S_d^+  } \| Q^{1/2} X S^{1/2}\|_\ast \geqslant \frac{1}{\kappa}  \Theta (X).
 \EEQ

We now consider two situations, where we do have such performance guarantees like \eq{rhotilde} and \eq{rho}. The first situation corresponds to constraints  on only diagonal elements of matrices (such as for $\U$ and $\V$ equal to $\ell_p$-norm balls for $p \geqslant 2$), and leads to dimension-independent guarantees; it was derived by~\cite{nesterov1998semidefinite}:
\begin{proposition}
\label{prop:diag2}
Assume $C_\U = \{ M \in \S_n^+ , \diag(M) \in \H_\U \}$ and $C_\V = \{ M \in \S_d^+ , \diag(M) \in \H_\V \}$ for some convex sets $\H_U \subset \rb^n_+$ and $\H_\V \subset \rb^d_+$, and that $\C_\U \cap \{ M, \ {\rm rank}(M)=1\} = \P_\U$, $\C_\V \cap \{ M, \ {\rm rank}(M)=1\} = \P_\V$.
Then \eq{rhotilde} and \eq{rho} are valid for $\kappa =  1 / ( 2 \sqrt{3}/\pi - 2 /3) < 2.3 $. If moreover, $\V$ is the unit $\ell_2$-ball, they are also valid for $\kappa = \displaystyle \sqrt{ {\pi}/{2}} < 1.3$.
\end{proposition}
\begin{proof}
We only consider the case where $\V$ is the unit disk, which gives the main idea of the proof of~\cite{nesterov1998semidefinite}. We consider a maximizer $M$ of \eq{relax2} and a vector $u =  D^{1/2} \sign(v)$ where $v$ is a random sample from a normal distribution with mean $0$ and covariance matrix $D^{-1/2} M D^{-1/2} $, with $D = \Diag(\diag(M))$.
We have, using standard arguments from~\cite{nesterov1998semidefinite,goemans1995improved}, $ uu^\top \in \C_\U \cap \{ {\rm rank}(M)=1\} = \P_\U$ and
\BEAS
\Theta^\circ(Y)^2 & \geqslant & \mathbb{E} \big[ u^\top YY^\top u \big]=
\mathbb{E} \tr D^{1/2} YY^\top  D^{1/2} \sign(v) \sign(v)^\top \\
& = & \frac{2}{\pi} \tr  D^{1/2} YY^\top D^{1/2} {\rm arcsin}  \big[D^{-1/2} M  D^{- 1/2} \big]\\
& \geqslant &  \frac{2}{\pi} \tr  D^{1/2} YY^\top  D^{1/2}    \big[
D^{-1/2} M  D^{- 1/2}\big] =  \frac{2}{\pi} \Tr Y^\top M Y  =  \frac{2}{\pi}   \sup_{ M \in \C_\U  }  {\tr  Y^\top M Y}.
\EEAS
The general case is also based on sampling and normalizing Gaussian random variables (see more details in~\cite{nesterov1998semidefinite}).
\end{proof}

\vspace*{.2500cm}

The other situation corresponds to situations where  the gauge functions $\gamma_\U$, $\gamma_\V$ and their polars have certain \emph{exact} quadratic representations (for example, for   $\U$ and $\V$ equal to $\ell_p$-norm balls for any $p \in [1,\infty]$).
\begin{proposition}
\label{prop:nd}
Assume that for all $w \in {\rm span}(\U)$ and $z \in {\rm span}(\V)$, we have the representations $\max\{ \gamma_\U^\circ(w),\gamma_\U^\circ(-w)\}^2 = \gamma_{\C_\U}^\circ(ww^\top)$ and $\max\{ \gamma_\V^\circ(z),\gamma_\V^\circ(-z)\}^2 = \gamma_{\C_\V}^\circ(zz^\top)$.
Then \eq{rhotilde} and \eq{rho} are valid for $\kappa = \min\{n,d\} $. If moreover, $\V$ is the unit disk, they are also valid for $\kappa = \sqrt{ \min\{n,d\}}$.
\end{proposition}
\begin{proof}
Our assumptions imply by duality that for all $w \in {\rm span}(\U)$,   $\gamma_{\U \cup (-\U)}(w)^2 = \inf_{ N \in \C_\U} w^\top N^{-1} w = G_\U(ww^\top)$, with $G_\U(M) = \inf_{ N \in \C_\U} M N^{-1} $, which is defined as, with $M = WW^\top \in \C_\U$, the minimum of $\tr B$ s.t.~$\Big(\! \begin{array}{cc} A  & W \\ W^\top & N  \end{array} \!\Big) \succcurlyeq 0$.
By taking $N = WW^\top$ and $A = \idm$, we obtain that $G_\U(M)  \leqslant {\rm rank}(M)$.

We first consider the case where $\V$ is the unit disk. We consider a solution $M$ of \eq{relax2} and a vector $u$ which is random sample from a normal distribution with mean $0$ and covariance matrix $ M$. By definition of $\C_\U$, $u$ must belong to ${\rm span}(\U)$. Let $r \leqslant \min\{n,d\}$ be the rank of $M$.
We have: $ \mathbb{E} [ u^\top  YY^\top u ] = \tr Y^\top M Y$. Moreover, by Jensen's inequality (and concavity of $G_\U$):
$$
\mathbb{E} \gamma_{{\U \cup (-\U)}}(u)^2 =    \mathbb{E} G_\U(uu^\top)  \leqslant  G_\U(  \mathbb{E} uu^\top)
\leqslant r \leqslant \min\{n,d\}.
$$
 Thus for $\alpha =  \frac{1}{\min\{n,d\}}   \sup_{ N \in \C_\U  }  {\tr  Y^\top N Y} =  \frac{1}{\min\{n,d\}}  \tr Y^\top M Y $, we have
$
\mathbb{E} \big[  u^\top YY^\top u - \alpha \gamma_{\U \cup (-\U)}(u)^2 \big] \geqslant  0.
$
Therefore, since $\gamma_\U$ is not uniformly equal to zero (because we have assumed $\U$ compact), there exists $u$ such that    $\gamma_\U(u) >0 $ and $u^\top YY^\top u - \alpha \gamma_\U(u)^2 \geqslant 0$, hence the claim.

For the general case, we sample $(u,v)$ from a normal distribution with mean zero and covariance matrix $
\Big( \! \begin{array}{cc} M  & Z \\ Z^\top & N  \end{array} \! \Big) $. We then have $\mathbb{E}[ u^\top Y v ]= \tr Z^\top Y$, and
$\mathbb{E} \gamma_{\U \cup (-\U)}(u)^2  \leqslant  \min\{n,d\}$, $\mathbb{E} \gamma_{\V \cup (-\V)}(v)^2  \leqslant  \min\{n,d\}$. We have,
with $\rho$ equal to the relaxed value obtained in \eq{relaxTT}:
$$
\mathbb{E} \bigg[
u^\top Y v - \frac{\rho}{ \min\{n,d\}} \big( \frac{1}{2} \gamma_{\U \cup (-\U)}(u)^2 + \frac{1}{2} \gamma_{\V \cup (-\V)}(v)^2 \big) 
\bigg] \geqslant 0.
$$
Thus, there exists $u,v$ such that $\gamma_\U(u) >0 $,  $\gamma_\V(v) >0 $ and $u^\top Y v - \frac{\rho}{ \min\{n,d\}} \big( \frac{1}{2} \gamma_\U(u)^2 + \frac{1}{2} \gamma_\V(v)^2 \big)$. This implies that  \eq{rhotilde} is valid for $\kappa =  \min\{n,d\}$.

Note that we may also obtain the same results without randomization. When $\gamma_\V = \| \cdot \|_2$, we have:
$$
\Theta^\circ(Y)^2 = \max_{ \|v \|_2  \leqslant 1 } \gamma_\U^\circ(Yv)^2 = 
\max_{ \|v \|_2  \leqslant 1 } \max_{ M \in \C_\U} v^\top Y^\top M Y v = \max_{M \in \C_\U} \lambda_{\max}(Y^\top M Y ).
$$
Thus by taking $v$ the largest eigenvector of $Y^\top M Y $ and $u \in \argmax_{u \in \U} u^\top Y v$, we have
$$
(u^\top Y v)^2 \geqslant v^\top Y^\top M Y  v  = \lambda_{\max}(Y^\top   M Y)  \geqslant 
\frac{1}{\min\{n,d\}}\tr Y^\top M Y.
$$
In the general case, we simply use
\BEAS
\Theta^\circ(Y)^2 & = & \max_{ v \in \V} \max_{M \in \C_\U} v^\top Y^\top M Y v  =
\max_{M \in \C_\U} \max_{N \in \C_\V} \  \max_{ v^\top N^{-1} v \leqslant 1 }  v^\top Y^\top M Y v  \\
& =  &
 \max_{ M \in \C_\U, \ N \in \C_\V} \lambda_{\max}(N^{1/2} Y^\top M Y N^{1/2})
\leqslant  \max_{ M \in \C_\U, \ N \in \C_\V} \| N^{1/2} Y^\top M ^{1/2} \|_{\ast}^2,
\EEAS
with an approximation ratio of $\min\{n,p\}$ between the operator norm and the nuclear norm.
\end{proof}

\vspace*{.2500cm}

The bound from the previous proposition is dimension-dependent. However, this is a non-trivial result. For example, with non-negativity constraints, it is not easy to have any guarantee (and our bounds do not). Indeed, For $\V$ being the $\ell_2$-norm and $\U$ the positive $\ell_1$-norm, then computing the polar corresponds to maximizing a quadratic form with positive constraints, which is notoriously difficult~\cite{berman2003completely}.

The two previous propositions lead to candidates for vectors $u$ and $v$, through  sampling from normal distributions with mean zero and covariance matrix equal to $M$ (or a normalized version thereof). Note that other possibilies are available, like taking a largest eigenvector  (such as done in the proof of Prop.~\ref{prop:nd}) and running the power method from \mysec{nonconvex} starting from any of the above candidates.

\paragraph{Approximation of $\gamma_{\P_\U}$.} Interestingly, computing $\gamma_{\P_\U}$, the gauge function considered in Prop.~\ref{prop:UU} which corresponds to factorizing positive semi-definite matrices with elements $uu^\top$, $u \in \U$, is harder than  computing the norm $\Theta$ for $\V$ equal to the unit $\ell_2$-ball. Indeed, the polar of $\gamma_{\P_\U}$ is defined by $\gamma_{\P_\U}^\circ(M) = \sup_{ u \in \U} (u^\top M u)_+ = 
\big( \sup_{ u \in \U}  u^\top M u \big)_+ $, and is finite for \emph{all} symmetric matrices, and in order to have an approximation ratio for $\gamma_{\P_\U}$ we need an approximation ratio for its polar (and not only for its polar restricted to $\S_n^+$). In the two types of guarantees above, the one based one diagonal variational representations may not be easily extended (note that this would notably imply that we would have a polynomial test of complete positivity, which is unlikely). However, if we have a variational formulation of the polar $\gamma_\U^\circ$, then the bound of Prop.~\ref{prop:nd} still applies (the two proof techniques may be easily extended). Note that we now impose that $\P_\U \subset \C_\U$.

\begin{proposition}
\label{prop:ndUU}
Assume that $\P_\U \subset \C_\U$, and that for all $w \in \rb^n$, we have
$\gamma_{\C_\U}^\circ(ww^\top) = \max\{ \gamma_\U^\circ(w),\gamma_\U^\circ(-w)\}^2 $. Then, for all $M \in \S_n^+$, $\gamma_{\P_\U}(M) \geqslant 
\gamma_{\C_\U}(M) \geqslant \frac{1}{n} \gamma_{\P_\U}(M)$.
\end{proposition}

\vspace*{.2500cm}

\subsection{Random sampling}
\label{sec:random}
An interesting alternative to semidefinite relaxations is to consider random sampling of vectors $u$ and $v$ and approximate $\Theta^\circ(Y)$ by the maximum of $u^\top Y v$ over these samples. This approach was considered by~\cite{obozinski-joint} in the context of the nuclear norm. In this section, for simplicity, we consider the case where $\V$ is the unit $\ell_2$-ball and $\U^\circ$  is compact (i.e., ${\rm dom}(\gamma_\U) = \rb^n$). In this context, only vectors~$u$ need to be sampled. We provide a positive result that if   sufficiently many vectors $u$ are sampled, then we have provable approximations of $\Theta^\circ$ and hence $\Theta$ (though with weak dimension-dependent ratios), as well as a negative result showing that in order to obtain an arbitrarily tight bound, an exponential number of samples is needed.

\paragraph{Upper bounds on $\Theta^\circ$.}
The following proposition provides an approximation ratio with high probability.
\begin{proposition}
\label{prop:rand}
Assume $\U$ and $\U^\circ$ are compact, and $\V $ is the unit $\ell_2$-ball. 
Consider $r$ independent and identically distributed samples $w_i \in \rb^n$ sampled from a standard normal distribution, and $u_i = w_i / \gamma_\U(w_i)$. If $r \geqslant 4n$, then, with probability greater than $1-e^{-r/50}$, we have, with $\displaystyle \kappa = 4  {\sqrt{n}}\Big(  {\max_{\Theta(Z)\leqslant 1} \|Z\|_F }  \Big) \Big( \max_{ \|u \|_2 \leqslant 1} \gamma_\U(u)
\Big)
 $:
\BEQ
\label{eq:random} \forall Y \in \rb^{n \times d}, \ 
 \Theta^\circ(Y)^2
\geqslant 
\max_{ i \in \{1,\dots,r\}} u_i^\top YY^\top u_i \geqslant \frac{1}{\kappa^2} \Theta^\circ(Y)^2.
\EEQ
\end{proposition}
\begin{proof}
A sufficient condition for \eq{random} is that 
$\forall Y \in \rb^{n \times d}, 
\max_{ i \in \{1,\dots,r\}} u_i^\top YY^\top u_i \geqslant  \big( \frac{1}{\kappa^2}{\max_{\Theta(Z)\leqslant 1} \tr Z^\top Z }  \big) \tr Y^\top Y $. A sufficient condition is that
 there exists $\eta \in \rb^r_+$ such that $1_r^\top \eta = 1$ and 
$
\lambda_{\min}\Big( \sum_{i=1}^r \eta_i u_i u_i^\top \Big) \geqslant \big( \frac{1}{\kappa^2}{\max_{\Theta(Z)\leqslant 1} \tr Z^\top Z }  \big)
$. Thus, by choosing $\eta_i \propto \gamma_\U(w_i)^2$, a sufficient condition is, for $\rho = 4 \sqrt{n}$:
$$
\lambda_{\min} \Big( \sum_{i=1}^r w_i w_i^\top \Big) \geqslant
\big( \frac{1}{\kappa^2}{\max_{\Theta(Z)\leqslant 1} \tr Z^\top Z }  \big) \Big( \max_{ \|u \|_2 \leqslant 1} \gamma_\U(u)^2 
\Big)
\sum_{i=1}^r \|w_i\|_2^2 = \frac{1}{\rho^2} \sum_{i=1}^r \|w_i\|_2^2,
$$
i.e.,   with $W = [w_1,\dots,w_r] \in \rb^{n \times r}$, $ f(W) = \sqrt{ \lambda_{\min}( WW^\top) } -  \frac{1}{\rho} \|W\|_F$.
The function $f$ is Lipschitz-continuous, i.e., such that $f(W')-f(W) \leqslant ( 1 + \frac{1}{4} ) \| W - W'\|_F$. Moreover, using~\cite[Theorem II.13]{davidson2001local},
$\mathbb{E} f(W) \geqslant \sqrt{r} - \sqrt{n} - \frac{1}{\rho} \sqrt{ rn}
= \sqrt{r} \big( 1 - \sqrt{n/r} - \frac{\sqrt{n}}{\rho} ) \geqslant \sqrt{r}/4$,
if $r \geqslant 4 n$ and $\rho = 4 \sqrt{n}$. Thus, by concentration of Lipschitz-continuous functions of standard Gaussian variables, we have
$
\mathbb{P} \big(
f(W) \leqslant 0
\big) \leqslant \exp \big( -  \frac{(\sqrt{r}/4)^2}{2 (5/4)^2} \big) = \exp(-r/50)
$, hence the result.
\end{proof}

\vspace*{.2500cm}

Note that the result above is rather weak as on top of the scaling in $\sqrt{n}$, the equivalence constants between $\Theta$ and $\| \cdot \|_F$ and between $\gamma_\U$ and $\| \cdot \|_2$ may be large as well (though finite because $\U^\circ$ is compact).

\paragraph{Upper bounds on approximation performance.}
We now show that even in the simplest case, (${\rm rank}(Y)=1$ and both $\U$ and $\V$ equal to  unit $\ell_2$-ball), the approximation ratio has to be dependent on dimensions. Using the same notation as Prop.~\ref{prop:rand}, we have,  using Lemma~\ref{lemma:beta} in Appendix~\ref{app:beta}, for any $(u,v) \in \rb^n \times \rb^d$, and $Y = uv^\top$
$\mathbb{E} \big[ \max_{ i \in \{1,\dots,r\}} u_i^\top YY^\top u_i  \big]\leqslant    \frac{4 \log r + 16}{n} \Theta^\circ(Y)^2$.
Thus, using random sampling may not give good approximation ratios, even in the simplest case, unless the number $r$ of samples is exponential in $n$.

\section{Obtaining decompositions from convex relaxations}
\label{sec:algorithms}
\label{sec:algorithm}
We have presented above semidefinite relaxations of $\Theta$ and $\Theta^\circ$. Much of earlier work~\cite{nesterov1998semidefinite,tal} and the previous section has been dedicated to obtaining for a certain $Y \in \rb^{n \times d}$, pairs $(u,v) \in \U \times \V$ such that $u^\top Y v \approx \Theta^\circ(Y)$. In this section, we focus on obtaining \emph{explicit} decomposition of a certain $X \in \rb^{n \times d}$ as  $X = \sum_{m=1}^r u_m v_m^\top $ such that $\sum_{m=1}^r \gamma_\U(u_m) \gamma_\V(v_m)$ is equal or better than the value of the relaxation, and thus getting a good approximation of $\Theta(X)$. We consider three main approaches.

\subsection{Singular value decomposition}
The first possibility is to  follow \mysec{QAB} and use solutions $Q,S$ of  
$$\sup_{Q \in \C_\U^\circ \cap \S_n^+, \ S \in \C_\V^\circ \cap \S_d^+ } \| Q^{1/2} X S^{1/2}\|_\ast$$
and decompose $Q = AA^\top$ with $A$ full rank (i.e., $A^\top A$ invertible), and $S=BB^\top$ with $B$ full rank. Then  $U =  (AA^\top)^{-1}  A G \Diag(s)^{1/2}$ and
   $V =  (BB^\top)^{-1} B H \Diag(s)^{1/2}$, are such that $X = UV^\top$ and $\frac{1}{2} \gamma_{\C_\U}(UU^\top) + \frac{1}{2} \gamma_{\C_\V}(VV^\top)$ is minimal (and equal to the value of the relaxation). However, for any orthogonal matrix $R$, $(UR,VR)$ is also such a pair, and   it is not possible to obtain a decomposition of $X$ such that the gauge functions $\gamma_\U$ and $\gamma_\V$ of all columns of $U$ and $V$ are small.

\subsection{Conditional gradient algorithms}
We may also find decompositions by approximately solving the following convex optimization problem (which is a generalized basis pursuit~\cite{chen} problem):
\BEQ
\label{eq:prox}
\min_{Z \in \rb^{n \times d} } \frac{1}{2  } \| X - Z \|_F^2 + \lambda \Theta(Z),
\EEQ
for $\lambda$ small enough, using convex optimization techniques that only access $\Theta$ through  computing $\Theta^\circ(Y) = \sup_{(u,v) \in \U \times \V } u^\top Y v$, and the associated minimizers. This is exactly what generalized conditional gradient algorithms can do~\cite{SGCG,zaidcg,zhang2012accelerated}. However, we need an algorithm which is robust to obtaining only approximate maximizers $(u,v)$, with potentially \emph{multiplicative} approximation guarantees for the computation of $\Theta^\circ$.

We consider the following algorithm started from $Z_0=0$, which iterates the following recursion, for $\rho_t = 2/(t+1)$, $t \geqslant 1$:
\BEAS
& (a)  \ \ \ & (u_{t-1},v_{t-1}) \in \arg\max_{u \in \U, \ v \in \V}  u^\top( X - Z_{t-1}) v  \\
& (b)  \ \ \ & \alpha_t = \arg\min_{\alpha \geqslant 0} \frac{1}{2} \big\| X - (1-\rho_t) Z_{t-1} - \rho_t \alpha  u_{t-1} v_{t-1}^\top \big\|_F^2 + \rho_t \lambda \alpha \\
& (c) \ \ \ & Z_t = (1-\rho_t) Z_{t-1} - \rho_t \alpha_t   u_{t-1} v_{t-1}^\top.
\EEAS
In Appendix~\ref{app:A}, we show that if we can find only approximate maximizers $(u,v)$ with approximation ratio $\kappa \geqslant 1$, then, if $X_\lambda$ is the unique solution of \eq{prox}, then we have
\BEAS
 \frac{1}{2} \| Z_t - X_\lambda \|_F^2
 & \leqslant  & \frac{1}{2} \| X - Z_t \|_F^2 + \lambda \Theta(Z_t) - 
\frac{1}{2} \| X - X_\lambda\|_F^2 - \lambda \Theta(X_\lambda) \\
& \leqslant & 
   \frac{2}{(t+1)}  \max  \{ 4 , \kappa^2  \} \Theta(X)^2  \max_{u \in \U } \|u\|_2^2 \max_{v \in \V } \|v\|_2^2 + \lambda ( \kappa - 1 ) \Theta(X),
\EEAS
and $Z_t$ is a positive linear combination of matrices $ u_{s-1} v_{s-1}^\top $, $s \leqslant t$, with a sum of coefficients which is less than $\Theta(Z_t)$. The previous inequality implies that
$
\Theta(Z_t) \leqslant \kappa \Theta(X) + O ( 1/(\lambda t) ).
$
Thus, when $\lambda$ is small enough and $t$ is large enough, we obtain an approximation of $\Theta(X)$ with approximation ratio which converges to a value less than $\kappa$.

\paragraph{Approaching $X$ from finite combinations of rank-one factors.}
When $\lambda=0$ in the algorithm above, then, every $X \in \rb^{n \times d}$ may be approximated up to distance $\varepsilon$ with a positive linear combination of $O( 1/ \sqrt{\varepsilon})$ rank-one factors, even if maximizing $u^\top Y v$ may only be done approximately for all $Y \in \rb^{n \times d}$. Moreover, the sum of coefficients is bounded by $\kappa \Theta(X) + O( \varepsilon)$.

Finding such decomposition by greedily and iteratively adding factors has been studied thoroughly in signal processing~\cite{mallat1993matching} and statistics~\cite{barron2008approximation}. In particular, if we assume that $\Theta$ is a norm, with our set of assumptions, the matching pursuit algorithm of~\cite{mallat1993matching} may obtain an $\varepsilon$-approximation of $X$ with $O( \log \frac{1}{\varepsilon})$ rank-one factors, however, while the norm of the coefficients is bounded, it is not related to the decomposition gauge function $\Theta(X)$. In Appendix~\ref{app:linear}, we show how optimizing over the scalar $\rho$ in the algorithm above leads to a similar result, while the sum of coefficients converge to a value which is less than $\kappa \Theta(X) $.
 
\subsection{Simplicial methods}
Conditional gradient algorithms to solve \eq{prox} may be extended by simply replacing step (b), which is the minimization over a half-line, by the minimization with respect to the cone generated by the already obtained rank-one matrices, i.e., 
\BEQ
\label{eq:b}
\beta^t = \arg\min_{ \beta \in \rb^{t-1} } \frac{1}{2} \Big\|
X - \sum_{s=1}^t \beta_s u_{s-1} v_{s-1}^\top 
\Big\|_F^2 + \lambda \sum_{s=1}^t \beta_s.
\EEQ
This approach is sometimes referred to as fully corrective~\cite{jaggi}
and is an instance of a simplicial method (see, e.g.,~\cite{bertsekas2011unifying}); typically, it requires much fewer iterations while the cost of each iteration is higher. Note that when the algorithm stops (i.e., there is no further progress in reducing the cost function), then we have solved \eq{prox} up to $\lambda(\kappa - 1) \Omega(x_\ast)$. An inbetween alternative is to optimize only over $\alpha$ and $\rho$. Note that the bound derived above also applies to these two extensions, which typically converge much quicker (see  examples in \mysec{simulations}).

Note that when $(u,v)$ is obtained from randomized rounding, the algorithm is related to what is proposed by~\cite{obozinski-joint}, which uses \emph{non-adaptive} random sampling for $u_{s-1} v_{s-1}^\top$.
Moreover, the algorithm may be accelerated by only storing only vectors $u_1,\dots,u_{t-1}$, and replacing the subproblem in \eq{b} by
\BEQ
\label{eq:v}
(v_0^t,\dots,v_{t-1}^t) = \arg\min_{ (v_0,\dots,v_{t-1}) \in \rb^{d \times t} } \frac{1}{2} \Big\|
X - \sum_{s=1}^t u_{s-1} v_{s-1}^\top
\Big\|_F^2 + \lambda \sum_{s=1}^t \gamma_\V(v_{s-1}).
\EEQ

\section{Simulations}
\label{sec:experiments}
\label{sec:simulations}
 
 In this section, we provide illustrations of the convex relaxations presented in the paper. We consider $\V$ the unit $\ell_2$-ball and $\U  = \{0,1\}^d$.
     We use $\C_\U = \{ M \in \S_n, \ M \succcurlyeq \diag(M) \diag(M)^\top, \ M \geqslant 0 , \  \diag(M) \leqslant 1 \}$, for which we may find an approximation guarantee of $\sqrt{ \pi/2}$ as follows: for $M \succcurlyeq 0$, maximizing $x^\top M x$ with respect to $x \in \{0,1\}^n$ may be done by maximizing $y^\top \Big( \! \begin{array}{cc} M & M1_n \\ 1_n^\top M & 1 \end{array} \! \Big) y$ with respect to $y \in \{-1,1\}^{n+1}$ such that $y_{n+1}=1$, which can be done using the usual semi-definite relaxation~\cite{nesterov1998semidefinite,goemans1995improved}, with an approximation ratio of $\pi/2$, and thus the gauge function $\Theta$ and its polar may be computed with an approximation ratio of~$\sqrt{\pi/2}$.

     \paragraph{Computation of polar gauge function.} We first compared several approaches to estimating
     $\Theta^\circ(Y) = \max_{ u \in \U,  \ v\in \V} u^\top Y v = \max_{ u \in \{0,1\}^n} \| Y^\top u\|_2$, for $Y$ a random matrix with independent and identically distributed components from a normal distribution with mean zero and variance one. We consider several strategies: (a) random sampling of $u$,  then  running the power method to convergence,  (b) sampling from the solution of the relaxed semi-definite program, with and without running the power method, and (c) taking the non-randomized approach described in the proof of Prop.~\ref{prop:nd}. In \myfig{pm}, we can see that (a) the performance of the semidefinite-relaxation, even without the power method, is typically much better than the guarantee,  (b) that sampling from the relaxed solution and then running the power method outperforms random initiatializations, and (c) the non-randomized rounding based on eigenvectors has a less stable behavior and sometimes performs  better.
     
     In \myfig{pm}, we report averaged value of the randomized rounding procedures. If we take the best values over more than a thousand samples, the obtained values of $u^\top Y v$ of all three schemes happens to be very close, with a slight advantage to the initializations of the power methods from the convex relaxation.
     
     \begin{figure}
     
     \vspace*{.34cm}

     \begin{center}
     \includegraphics[scale=.45]{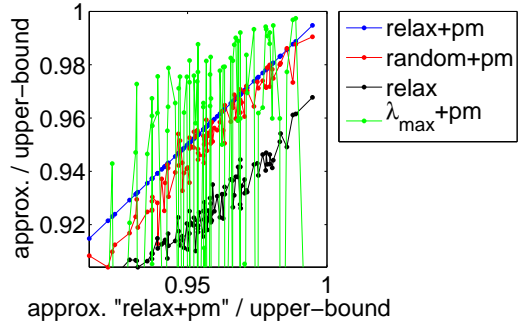}
     \end{center}
     
     \vspace*{-.34cm}
     
     \caption{Approximating the polar gauge function $\Theta^\circ(Y)$ for $Y$ standard random Gaussian matrix with $n=64$ and $d=32$, for 100 samples. For each sample, we compute the relaxed value (upper-bound) based on the semidefinite program and compute several approximations. The samples are ordered so that the performance of the randomized rounding followed by the power method (``relax+pm'') is increasing.}
     \label{fig:pm}
     \end{figure}

     \paragraph{Finding decomposition.}
     We aim to solve the problem in \eq{prox} for $\lambda=10^{-4}$, and consider four approaches: (a) using the semidefinite relaxation to obtain a lower-bound (with no explicit decomposition), (b) using the conditional gradient algorithm, (c) using a simplicial method (the regular version based on \eq{b} or the one adapted to storing only values of $u$ in \eq{v}), (d) using alternating optimization and (e) random sampling (\mysec{random}). In \myfig{prox}, we compare these algorithms in two situations, one where the relaxation is tight (left: $n=32$, $d=1$) and one where it is not tight (right: $n=32$, $d=16$). The simplicial algorithms are the faster to converge, with a clear advantage to the one that stores only the vectors $u$. The random selection procedure starts slow but eventually catches up, but never reaches the objective function of adaptive methods.
     
     Finally, in \myfig{dec}, we compare the result of using the simplicial method to obtain a decomposition to a simple alternating optimization method. We see that the convex relaxation outperforms significantly the non-convex approach.

     \begin{figure}
     
     \vspace*{.34cm}

     \begin{center}
     \includegraphics[scale=.45]{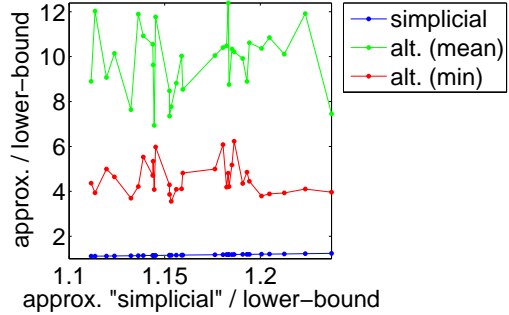}
     \end{center}
     
     \vspace*{-.34cm}
     
     \caption{Approximating the  gauge function $\Theta(X)$ for $X$ standard random Gaussian matrix with $n=32$ and $d=16$, for 32 samples. For each sample, we compute the relaxed value (lower-bound) based on the semidefinite program and compute   several approximations. }
     \label{fig:dec}
     \end{figure}

  \begin{figure}
  
  \vspace*{.34cm}

     \begin{center}
     \includegraphics[scale=.45]{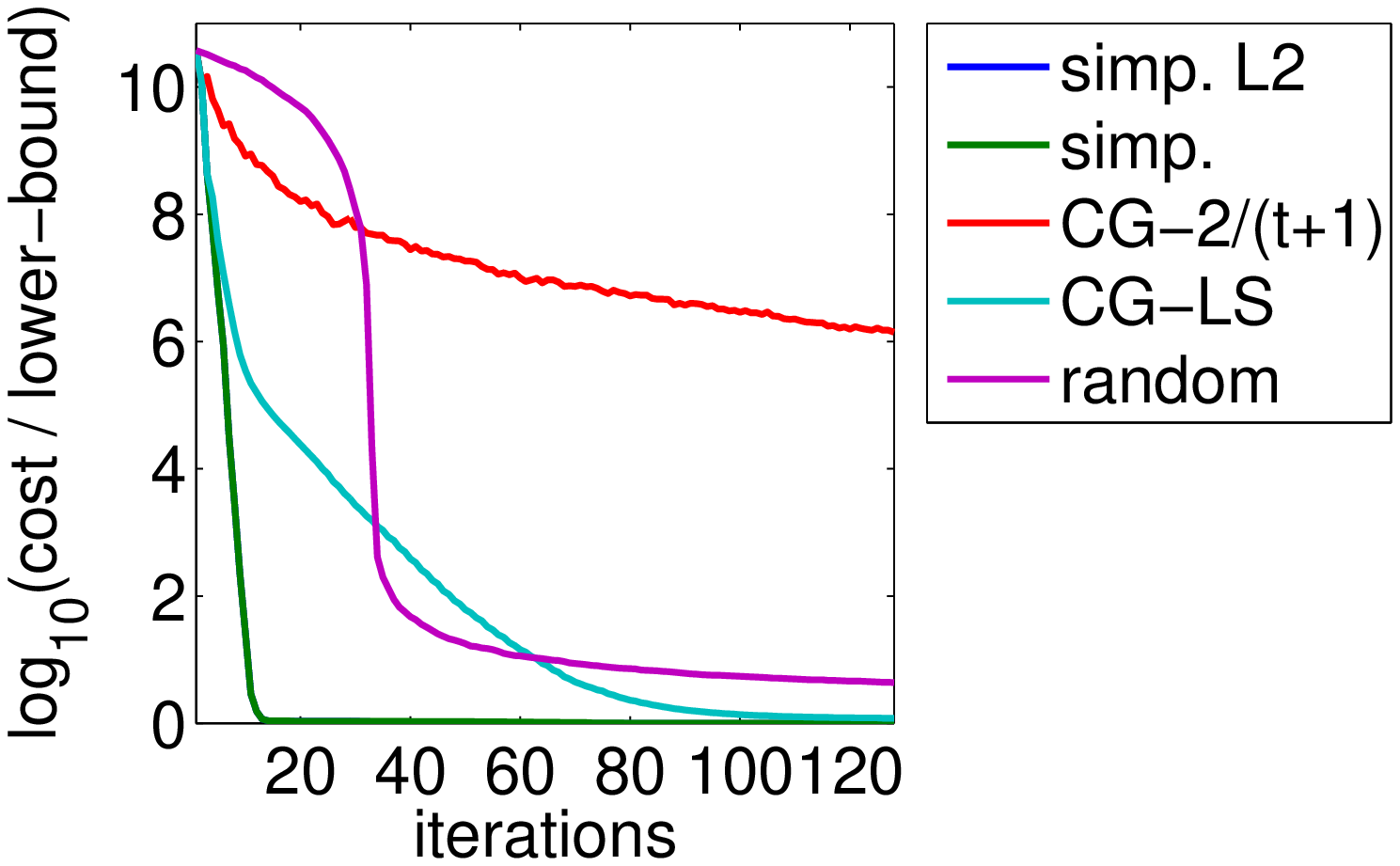} \hspace*{.25cm}
     \includegraphics[scale=.45]{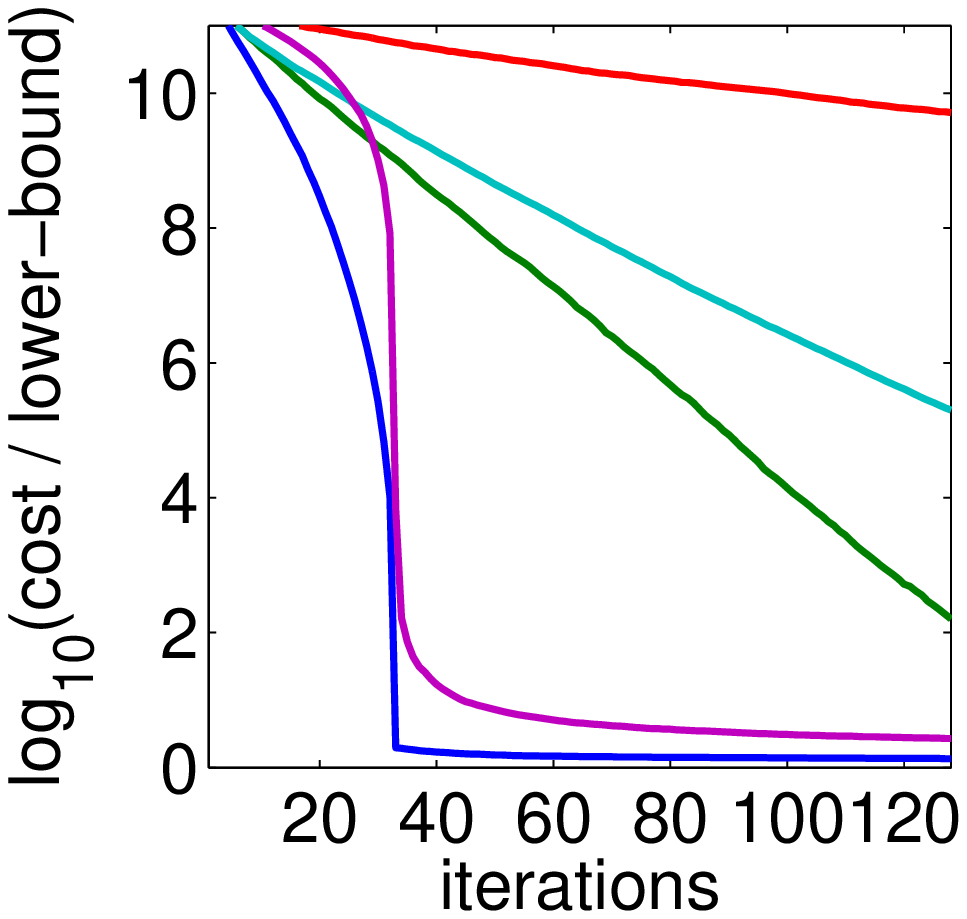}
     \end{center}
     
     \vspace*{-.34cm}
     
     \caption{Solving the proximal problem for two random matrices $Y$, with $n=32$ and $d=1$ (left) and $d=16$ (right).  Comparison of several algorithms. See text for details.}
     \label{fig:prox}
     \end{figure}

     \section{Conclusion}
     We have presented a general framework for structured matrix decompositions based on gauge functions and semi-definite programming. Emphasis was put on situations where the rank-one factors belong to potentially non-convex and non-centrally symmetric sets. A series of algorithms and relaxations have been presented for a variety of structures. 

Our limited experiments have focused on $\{0,1\}$-valued factors. It would be worth studying more precisely recovery guarantees and potentially more scalable algorithms for such cases, in particular for additional constraints such as cardinality or submodularity. Similarly, our general framework leads to altenatives to non-convex algorithms for non-negative matrix factorization, for which it may be possible to show tightness under certain assumptions similar to~\cite{arora2012computing}.

Moreover, it would be interesting to see if techniques designed to adaptively reduce the rank for nuclear-norm penalized problems~\cite{zhang2010analysis} may be extended to our setting as well in order to enforce a smaller number of factors. Finally, the non-convex approaches (power method and alternating minimization) tend to work well in practice, and finding sufficient conditions or random instances (see, e.g.,~\cite{deshpande2013finding} for such work for sparse principal component analysis) where they provably behave well would provide valuable additional insights into the problem of structured matrix factorizations.

\section*{Acknowledgments}
This  work was supported by the European Research Council (SIERRA project 239993). The author thanks Guillaume Obozinski and Alexandre d'Aspremont for fruitful discussions related to this work.

 \appendix
 
 \section{Generalized conditional gradient with approximate oracle}
 \label{app:A}
 We consider a convex function $f$ on $\rb^d$, with $L$-Lipschitz continuous gradients, a gauge function $\gamma_\C$ on $\rb^d$, with $\C$ a (potentially non convex) closed set that contains the origin, and the following optimization problem $\min_{x \in \rb^d}  f(x) + \lambda \gamma_\C(x)$. We assume that (a) the global minimum $x_\ast$ is attained, (b) that we know a bound on $\gamma_\C(x_\ast)$, i.e., $\gamma_\C(x_\ast) \leqslant \omega$ and (c) we may maximize $x^\top y$ with respect to $x \in \C$ approximately, i.e., obtain $\tilde{x} \in \C$   such that
 $\gamma_\C^\circ(y) \geqslant \tilde{x}^\top y \geqslant \frac{1}{\kappa} \gamma_\C^\circ(y)$, for a fixed $\kappa \geqslant 1$. We consider the following algorithm, started from any $x_0$ such that $\gamma_\C(x_0) \leqslant \kappa \omega$, for $t \geqslant 1$, and $\rho_t = 2 / (t+1)$:
 \BEAS
 (a) \ \ & &  \tilde{x}_{t-1}  \mbox{ approximate maximizer  of } -f'(x_{t-1})^\top x  \mbox{ with respect to } x \in \C,  \\
 (b) \ \ & & \alpha_t  \in   \arg \min_{ \alpha \in [ 0 , \kappa \omega] } f\big(
  ( 1 - \rho_t  ) x_{t-1} +  \rho_t \alpha \tilde{x}_{t-1}
 \big)
   + \lambda \rho_t  \alpha ,
  \\
  \mbox{ or } (b') \ \ & & \alpha_t  \in   \arg \min_{ \alpha \in [ 0 , \kappa \omega] } -   f'(x_{t-1})^\top \big(
  x_{t-1} -  \alpha \tilde{x}_{t-1}
 \big) + \frac{L \rho_t}{2}  \big\|
  x_{t-1} -  \alpha \tilde{x}_{t-1}
 \big\|_2^2 + 
    \lambda  \alpha ,
  \\
  (c) \ \ & & x_t = ( 1 - \rho_t  ) x_{t-1} +  \rho_t \alpha_t \tilde{x}_{t-1}.
 \EEAS
 The algorithm generates a sequence of vectors $\tilde{x}_{t-1} \in \C$ and a vector 
 $$x_{t} =   \sum_{u=1}^t \bigg[
 \prod_{s=u+1}^t ( 1 - \rho_s)
 \bigg] \rho_u \alpha_u \tilde{x}_{u-1} 
 =   \sum_{u=1}^t \frac{ 2 u  }{t(t+1)}  \alpha_u \tilde{x}_{u-1}.
 $$
 Note that since $\tilde{\alpha}_{t-1} \in [0, \kappa \omega]$ for all $t$, we always have $\gamma_\C(x_{t}) \leqslant \kappa \omega$.
 We now show that
 \BEAS
 f(x_{t}) + \lambda \gamma_\C(x_t) & \leqslant & f(x_{t}) + \lambda \sum_{u=1}^t \frac{ 2 u  }{t(t+1)}  \alpha_u \\
 & \leqslant &  f(x_\ast) + \lambda \gamma_\C(x_\ast) + \frac{2 L \omega^2 \kappa^2 \max_{u \in \U}\| u\|^2 }{t+1} + \lambda (\kappa-1) \gamma_\C(x_\ast) .
 \EEAS
 Let $g_t = f(x_{t}) + \lambda \sum_{u=1}^t \frac{ 2 u  }{t(t+1)} \alpha_u$.
 We have, for any $\alpha \in [0,\kappa \omega]$, and $ \tilde{x} \in {\rm hull}(\C)$:
 \BEAS
 \!\! g_t & = & f(x_{t}) + (1-\rho_t) \big[ g_{t-1} - f(x_{t-1}) \big] + \lambda \alpha_t \rho_t  \\
 &   =  & f\big(
  ( 1 - \rho_t  ) x_{t-1} + \rho_t \alpha_t \tilde{x}_{t-1}
 \big) + (1-\rho_t) \big[ g_{t-1} - f(x_{t-1}) \big] + \lambda \alpha_t \rho_t  \\
  &   \leqslant  & f\big(x_{t-1})  - \rho_t  f'(x_{t-1})^\top \big(
  x_{t-1} -  \alpha \tilde{x}_{t-1}
 \big) + \frac{L \rho_t^2}{2}  \big\|
  x_{t-1} -  \alpha \tilde{x}_{t-1}
 \big\|_2^2  + (1-\rho_t) \big[ g_{t-1} - f(x_{t-1}) \big] + \lambda \alpha \rho_t  \\
 & & \mbox{ by definition of $\alpha_t$}, \\
 &   \leqslant  & f\big(x_{t-1})  - \rho_t  f'(x_{t-1})^\top \big(
  x_{t-1} -  \alpha \tilde{x}_{t-1}
 \big) + \frac{L \rho_t^2}{2}  \kappa^2 \omega^2 \max_{u \in \U}\| u\|^2  + (1-\rho_t) \big[ g_{t-1} - f(x_{t-1}) \big] + \lambda \alpha  \rho_t  
\\
 &   \leqslant  & f\big(x_{t-1})  - \rho_t  f'(x_{t-1})^\top \big(
  x_{t-1} -  \frac{\alpha}{\kappa} \tilde{x}
 \big) + \frac{L \rho_t^2}{2}  \kappa^2 \omega^2 \max_{u \in \U}\| u\|^2  + (1-\rho_t) \big[ g_{t-1} - f(x_{t-1}) \big] + \lambda \alpha  \rho_t  
\\
& & \mbox{ using the approximate optimality of } \tilde{x}_{t-1}, \\
 &   \leqslant  & f\big(x_{t-1})  - \rho_t  f'(x_{t-1})^\top \big(
  x_{t-1} -  x_\ast
 \big) + \frac{L \rho_t^2}{2}  \kappa^2 \omega^2 \max_{u \in \U}\| u\|^2  + (1-\rho_t) \big[ g_{t-1} - f(x_{t-1}) \big] + \lambda \kappa \gamma_\C(x_\ast) \rho_t  
\\
& & \mbox{ using } \alpha = \kappa \gamma_\C(x^\ast) \mbox{ and } \tilde{x} = x_\ast / \gamma_\C(x_\ast) \in {\rm hull}(\C), \\
&   \leqslant  & f\big(x_{t-1})  - \rho_t  [ f(x_{t-1}) - f(x_\ast) ] + \frac{L \rho_t^2}{2}  \kappa^2 \omega^2 \max_{u \in \U}\| u\|^2  + (1-\rho_t) \big[ g_{t-1} - f(x_{t-1}) \big] + \lambda \kappa \gamma_\C(x_\ast) \rho_t  
\\
&   =  & (1 - \rho_t) g_{t-1} + \rho_t f(x_\ast)  + \frac{L \rho_t^2}{2}  \kappa^2 \omega^2 \max_{u \in \U}\| u\|^2    +\lambda  \kappa \gamma_\C(x_\ast) \rho_t   ,
\EEAS
leading to
\BEAS
g_t - f(x_\ast) - \lambda\gamma_\C(x_\ast)
& \leqslant & (1 - \rho_t)  \big[ g_{t-1} - f(x_\ast) - \lambda\gamma_\C(x_\ast)
\big] + \frac{L \rho_t^2}{2}  \kappa^2 \omega^2 \max_{u \in \U }\| u\|^2 + \lambda (\kappa-1) \gamma_\C(x_\ast) \rho_t .
 \EEAS
 This implies by recursion (see~\cite{SGCG} for details) that
 $$ g_t - f(x_\ast) - \lambda\gamma_\C(x_\ast)
\leqslant \frac{2L}{(t+1)}  \kappa^2 \omega^2 \max_{u \in \U }\| u\|^2 + \lambda (\kappa-1) \gamma_\C(x_\ast) .
$$

 If moreover, $f$ is $\mu$-strongly convex and the line search is performed exactly (and without a bound on $\alpha$), then we may show a different bound. Up to the oracle with multiplicative approximation guarantees, the algorithm is then the same as what is proposed by~\cite{zaidcg,zhang2012accelerated}, but our results use a slightly different set of assumptions.
 
 First, we show that $f(x_t)$ remains bounded. Indeed, we have $f(x_t) \leqslant f((1-\rho_t)x_{t-1}) \leqslant ( 1-\rho_t) f(x_{t-1}) + \rho_t f(0)$, which leads to $f(x_t) \leqslant f(0)$ for all $t \geqslant 1$. This implies that $f(0) \geqslant f(0) + x_t^\top f'(0) + \frac{\mu}{2} \| x_t\|_2^2 \geqslant - \|x_t\|_2 \| f'(0) \|_2+ \frac{\mu}{2} \| x_t\|_2^2   $, leading to $\| x_t \|_2 \leqslant 2 \| f'(0)\|_2/ \mu$.
 
We may then derive a different recursion, leading to 
 $$ g_t - f(x_\ast) - \lambda\gamma_\C(x_\ast)
\leqslant \frac{2L}{(t+1)}  \max \bigg\{ \frac{2 \| f'(0)\|_2}{\mu} , \kappa  \omega  \max_{u \in \C}\| u\| \bigg\}^2 + \lambda (\kappa-1) \gamma_\C(x_\ast) .
$$

When applied to  $f(x) = \frac{1}{2} \| x - y\|_2^2$, we obtain a decaying factor of
$$ \frac{2}{(t+1)}  \max  \{ 4 , \kappa^2  \} \gamma_\C(y)^2  \max_{u \in \C}\| u\|^2 .$$
In this case, our algorithm is strongly related to the relaxed greedy algorithm of~\cite{barron2008approximation} and our analysis provides an explicit link between these algorithms and basis pursuit~\cite{chen}.

 \subsection{Linearly convergent conditional gradient for strongly convex functions with inexact oracle}
 \label{app:linear}
 We now assume that $f$ is $\mu$-strongly convex and that $f$ has a global minimum attained at $x_\ast$ such that $\gamma_\C(x_\ast) \leqslant  \omega$, and that we have a $\kappa$-approximate oracle for maximizing linear functions on $\C$. We consider the algorithm:
  \BEAS
 (a) \ \ & &  \tilde{x}_{t-1}  \mbox{ approximate maximizer  of } -f'(x_{t-1})^\top x  \mbox{ with respect to } x \in \C,  \\
 (b) \ \ & & (\alpha_t,\rho_t)  \in   \arg \min_{ \alpha \in [ 0 , \kappa \omega + \kappa\varepsilon] , \rho \in [0,1] } f\big(
  ( 1 - \rho ) x_{t-1} +  \rho \alpha \tilde{x}_{t-1}
 \big)
    ,
  \\
  \mbox{ or } (b') \ \ & &  (\alpha_t,\rho_t)   \in   \arg \min_{ \alpha \in [ 0 , \kappa \omega  + \kappa\varepsilon]  , \rho \in [0,1] } -   f'(x_{t-1})^\top \big(
  x_{t-1} -  \alpha \tilde{x}_{t-1}
 \big) + \frac{L \rho}{2}  \big\|
  x_{t-1} -  \alpha \tilde{x}_{t-1}
 \big\|_2^2 
     ,
  \\
  (c) \ \ & & x_t = ( 1 - \rho_t  ) x_{t-1} +  \rho_t \alpha_t \tilde{x}_{t-1}.
 \EEAS
The algorithm generates a sequence of  vectors $\tilde{x}_{t-1} \in \C$ and a vector 
 $$x_{t} =   \sum_{u=1}^t \bigg[
 \prod_{s=u+1}^t ( 1 - \rho_s)
 \bigg] \rho_u \alpha_u \tilde{x}_{u-1} 
 $$ 
 such that $ \sum_{u=1}^t \bigg[
 \prod_{s=u+1}^t ( 1 - \rho_s)
 \bigg] \rho_u \alpha_u \leqslant \kappa \omega  + \kappa\varepsilon$. We have, for any $\alpha \in [0,\kappa \omega + \kappa \varepsilon]$, $\rho \in [0,1]$, and $ \tilde{x} \in {\rm hull}(\C)$:
\BEAS
 f(x_t)  
  &   \leqslant  & f (x_{t-1})  - \rho  f'(x_{t-1})^\top \big(
  x_{t-1} -  \alpha \tilde{x}_{t-1}
 \big) + \frac{L \rho^2}{2}  \big\|
  x_{t-1} -  \alpha \tilde{x}_{t-1}
 \big\|_2^2    \\
 & & \mbox{ by definition of $\alpha_t$ and $\rho_t$}, \\
 &   \leqslant  & f\big(x_{t-1})  - \rho  f'(x_{t-1})^\top \big(
  x_{t-1} -  \alpha \tilde{x}_{t-1}
 \big) + \frac{L \rho^2}{2} \kappa^2 (   \omega  + \varepsilon)^2 \max_{u \in \C}\| u\|^2   
\\
 &   \leqslant  & f\big(x_{t-1})  - \rho  f'(x_{t-1})^\top \big(
  x_{t-1} -  \frac{\alpha}{\kappa} \tilde{x}
 \big) + \frac{L \rho^2}{2}  \kappa^2 (    \omega  + \varepsilon)^2 \max_{u \in \C}\| u\|^2  
\\
& & \mbox{ using the approximate optimality of } \tilde{x}_{t-1}, \\
 &   \leqslant  & f\big(x_{t-1})  - \rho  f'(x_{t-1})^\top \big(
  x_{t-1} -  x_\ast + \frac{ \varepsilon}{\gamma_\C(x_{t-1} - x_\ast)}(x_{t-1}-x_\ast)
 \big) + \frac{L \rho^2}{2}  \kappa^2 (   \omega  + \varepsilon)^2 \max_{u \in \C}\| u\|^2   \\
& & \mbox{ using } \alpha = \kappa \gamma_\C(\tilde{x}) \mbox{ and } \tilde{x} \propto x_\ast - \frac{ \varepsilon}{\gamma_\C(x_{t-1} - x_\ast)}(x_{t-1}-x_\ast) ,\\
&   \leqslant  & f\big(x_{t-1})  - \rho  [ f(x_{t-1}) - f(x_\ast) ] 
\Big(
1 + \frac{\varepsilon}{\gamma_\C(x_{t-1} - x_\ast)}
\Big)+ \frac{L \rho^2}{2}  \kappa^2 \omega^2 \max_{u \in \C}\| u\|^2   .
\EEAS
Moreover, we have 
$\displaystyle  f (x_{t-1})  - f(x_\ast) \geqslant \frac{\mu}{2} \| x_{t-1} - x_\ast\|_2^2 \geqslant \frac{\mu}{2} \frac{ \gamma_\C(x_{t-1}-x_\ast)^2 }{ \max_{\|u\|_2=1} \gamma_\C(u)^2}$, leading to, with $\Delta_t = f (x_{t-1})  - f(x_\ast) $,
\BEAS
\Delta_t & \leqslant &  \Delta_{t-1} - \rho \bigg[ \Delta_{t-1} + \frac{  \varepsilon  \sqrt{\mu} \Delta_{t-1}^{1/2}} {\sqrt{2} \max_{\|u\|_2=1} \gamma_\C(u)}
\bigg] + \frac{L \rho^2}{2}  \kappa^2 \omega^2 \max_{u \in \C}\| u\|^2    .
\EEAS
If $ \displaystyle \bigg[  \Delta_{t-1}  + \frac{  \varepsilon  \sqrt{\mu} \Delta_{t-1}^{1/2}} {\sqrt{2} \max_{\|u\|_2=1} \gamma_\C(u)}
\bigg]  \frac{1}{ L \kappa^2 \omega^2 \max_{u \in \C}\| u\|^2} < 1$, we have a minimizer $\rho \in [0,1)$, and
\BEAS
\Delta_t & \leqslant & \Delta_t - \frac{1}{2} \bigg[  \frac{  \varepsilon  \sqrt{\mu} \Delta_{t-1}^{1/2}} {\sqrt{2} \max_{\|u\|_2=1} \gamma_\C(u)}
\bigg]^2 \frac{1}{ L \kappa^2 \omega^2 \max_{u \in \C}\| u\|^2} \\
& \leqslant & \Delta_{t-1} \bigg( 1 - \frac{   \varepsilon^2 \mu}{ 4 L \kappa^2 \omega^2} \frac{1}{
\max_{\|u\|_2=1} \gamma_\C(u)^2 \times  \max_{u \in \C}\| u\|^2} \bigg) .
\EEAS
Otherwise, we have
\BEAS
\Delta_t & \leqslant & \Delta_{t-1} - \frac{1}{2}  \bigg[ \Delta_{t-1} + \frac{  \varepsilon  \sqrt{\mu} \Delta_{t-1}^{1/2}} {\sqrt{2} \max_{\|u\|_2=1} \Omega(u)}
\bigg] 
\leqslant \frac{1}{2} \Delta_{t-1}
\EEAS

Thus, with $\displaystyle \tau = \min \bigg\{
\frac{1}{2}, \frac{   \varepsilon^2 }{4  \kappa^2 \omega^2} \frac{\mu}{L} \frac{1}{
\max_{\|u\|_2=1} \gamma_\C(u) \times  \max_{u \in \C}\| u\|^2}  \bigg\}
$, we have $\Delta_t \leqslant ( 1- \tau) \Delta_{t-1}$, and hence a linear convergence rate.

\section{Maximum of beta random variables}
\label{app:beta}
Given $u \in \rb^n$ such that $\|u\|_2=1$, our goal is to upper-bound $\displaystyle \mathbb{E} \big[ \max_{ i \in \{1,\dots,r\}} \frac{ (w_i^\top u)^2}{w_i^\top w_i}  \big]$, for $w_i$ sampled i.i.d~from a standard normal distribution. Using the representation of Beta random variables as ratios of independent Gamma variables, each $\frac{ (w_i^\top u)^2}{w_i^\top w_i}$ is Beta-distributed with parameters $(\frac{1}{2}, \frac{n-1}{2})$. The following lemma provides a bound on the expectation of maxima of independent Beta variables.

\begin{lemma}
\label{lemma:beta} Let $X_i$, $i=1,\dots,r$, be $r$ i.i.d.~Beta random variables with parameter $(\frac{1}{2},\frac{n-1}{2})$, $n \geqslant 2$.
Then $\displaystyle \mathbb{E} \big[ \max_{ i \in \{1,\dots,r\}}  X_i \big] \leqslant \frac{4 \log r + 16}{n}.$
\end{lemma}
\begin{proof}
We have $\mathbb{E} X^k = \prod_{j=0}^{k-1} \frac{1 + 2j}{n+2j}$. For all $k>0$, we have
$
\mathbb{E} X^k \leqslant \prod_{j=0}^{k-1} \frac{2 + 2j}{n} = \frac{2^k k!}{n^k}
$, while for $k \geqslant \lfloor n/2 \rfloor $, we have
$
\mathbb{E} X^k \leqslant \prod_{j=0}^{ \lfloor n/2 \rfloor - 1} \frac{1 + 2j}{n+2j} \leqslant 2^{- \lfloor n/2 \rfloor }
$. This leads to, for $t \leqslant n/2$,
$$ \mathbb{E} e^{t X} \leqslant \sum_{k=0}^{ \lfloor n/2 \rfloor }  \big( \frac{2 t}{n} \big)^k + 2^{- \lfloor n/2 \rfloor } \sum_{k =  \lfloor n/2 \rfloor +1}^ \infty \frac{t^k}{k!} \leqslant \frac{1}{1-2t/n} +   2^{1/2-n/2} e^t.
$$
Using standard results from probability (see, e.g.,~\cite{boucheron2013concentration}), we get, with $t = n/4$:
$$
\mathbb{E} \big[ \max_{ i \in \{1,\dots,r\}}  X_i \big]
\leqslant \frac{\log r  + \mathbb{E} e^{t X} }{t}
\leqslant \frac{\log r + 2 + 2^{1/2-n/2} e^{n/4}}{n/4} \leqslant \frac{ 4\log r +  16 }{n}.
$$

\end{proof}
   \bibliographystyle{plain}
\bibliography{decomposition}

\end{document}

%% file: macros.tex
   \usepackage[dvips]{graphicx}
\usepackage{amssymb,amsmath,color}

\newcommand{\BEAS}{\begin{eqnarray*}}
\newcommand{\EEAS}{\end{eqnarray*}}
\newcommand{\BEA}{\begin{eqnarray}}
\newcommand{\EEA}{\end{eqnarray}}
\newcommand{\BEQ}{\begin{equation}}
\newcommand{\EEQ}{\end{equation}}
\newcommand{\BIT}{\begin{itemize}}
\newcommand{\EIT}{\end{itemize}}
\newcommand{\BNUM}{\begin{enumerate}}
\newcommand{\ENUM}{\end{enumerate}}
\newcommand{\BA}{\begin{array}}
\newcommand{\EA}{\end{array}}
\newcommand{\diag}{\mathop{\rm diag}}
\newcommand{\Diag}{\mathop{\rm Diag}}

\newcommand{\argmax}{\mathop{\rm argmax}}

\newcommand{\Tr}{\mathop{ \rm tr}}
\newcommand{\tr}{\mathop{ \rm tr}}
\newcommand{\sign}{\mathop{ \rm sign}}
\newcommand{\idm}{I}
\newcommand{\rb}{\mathbb{R}}
\newcommand{\BlackBox}{\rule{1.5ex}{1.5ex}}  

\newcommand{\mysec}[1]{Section~\ref{sec:#1}}
\newcommand{\eq}[1]{Eq.~(\ref{eq:#1})}
\newcommand{\myfig}[1]{Figure~\ref{fig:#1}}

\def \C{ {\mathcal{C}}}
\def \D{ {\mathcal{D}}}
\def \P{ {\mathcal{P}}}
\def \S{ {\mathcal{S}}}
\def \K{ {\mathcal{K}}}
\def \L{ {\mathcal{L}}}

\def \H{ {\mathcal{H}}}

\def \U{ {\mathcal{U}}}
\def \V{ {\mathcal{V}}}

%% file: struct_decomposition_hal.bbl
\begin{thebibliography}{10}

\bibitem{Aharon2006}
M.~Aharon, M.~Elad, and A.~M. Bruckstein.
\newblock The {K-SVD}: An algorithm for designing of overcomplete dictionaries
  for sparse representations.
\newblock {\em IEEE Transactions on Signal Processing}, 54(11):4311--4322,
  2006.

\bibitem{alon2006approximating}
N.~Alon and A.~Naor.
\newblock Approximating the cut-norm via {G}rothendieck's inequality.
\newblock {\em SIAM Journal on Computing}, 35(4):787--803, 2006.

\bibitem{arora2012computing}
S.~Arora, S.~Ge, R.~Kannan, and A.~Moitra.
\newblock Computing a nonnegative matrix factorization--provably.
\newblock In {\em Proceedings of the Symposium on Theory of Computing (STOC)},
  pages 145--162, 2012.

\bibitem{bach2011learning}
F.~Bach.
\newblock Learning with submodular functions: A convex optimization
  perspective.
\newblock {\em Arxiv preprint arXiv:1111.6453}, 2011.

\bibitem{SGCG}
F.~Bach.
\newblock Duality between subgradient and conditional gradient methods.
\newblock Technical Report 00757696, HAL, 2013.

\bibitem{fot}
F.~Bach, R.~Jenatton, J.~Mairal, and G.~Obozinski.
\newblock Optimization with sparsity-inducing penalties.
\newblock {\em Foundations and Trends{\textregistered} in Machine Learning},
  4(1):1--106, 2011.

\bibitem{bach2008convex}
F.~Bach, J.~Mairal, and J.~Ponce.
\newblock Convex sparse matrix factorizations.
\newblock Technical Report 00345747, HAL, 2008.

\bibitem{barron2008approximation}
A.~R. Barron, A.~Cohen, W.~Dahmen, and R.~A. DeVore.
\newblock Approximation and learning by greedy algorithms.
\newblock {\em The annals of Statistics}, 36(1):64--94, 2008.

\bibitem{ben2009robust}
A.~Ben-Tal, L.~El~Ghaoui, and A.~Nemirovski.
\newblock {\em Robust optimization}.
\newblock Princeton University Press, 2009.

\bibitem{tal}
A.~Ben-Tal and A.~Nemirovski.
\newblock On approximating matrix norms.
\newblock Technical report, Technion, 2011.

\bibitem{berman2003completely}
A.~Berman and N.~Shaked-Monderer.
\newblock {\em Completely positive matrices}.
\newblock World Scientific, 2003.

\bibitem{bertsekas2011unifying}
D.~P. Bertsekas and H.~Yu.
\newblock A unifying polyhedral approximation framework for convex
  optimization.
\newblock {\em SIAM Journal on Optimization}, 21(1):333--360, 2011.

\bibitem{borwein2006caa}
J.~M. Borwein and A.~S. Lewis.
\newblock {\em Convex Analysis and Nonlinear Optimization: Theory and
  Examples}.
\newblock Springer, 2006.

\bibitem{boucheron2013concentration}
S.~Boucheron, G.~Lugosi, and P.~Massart.
\newblock {\em Concentration Inequalities: A Nonasymptotic Theory of
  Independence}.
\newblock Oxford University Press, 2013.

\bibitem{boyd1974power}
D.~W. Boyd.
\newblock The power method for $\ell_p$-norms.
\newblock {\em Linear Algebra and its Applications}, 9:95--101, 1974.

\bibitem{burer2003nonlinear}
S.~Burer and R.~D.~C. Monteiro.
\newblock A nonlinear programming algorithm for solving semidefinite programs
  via low-rank factorization.
\newblock {\em Mathematical Programming}, 95(2):329--357, 2003.

\bibitem{chandrasekaran2012convex}
V.~Chandrasekaran, B.~Recht, P.~A., and A.~S. Willsky.
\newblock The convex geometry of linear inverse problems.
\newblock {\em Foundations of Computational Mathematics}, 12(6):805--849, 2012.

\bibitem{chen}
S.~S. Chen, D.~L. Donoho, and M.~A. Saunders.
\newblock Atomic decomposition by basis pursuit.
\newblock {\em SIAM Journal on Scientific Computing}, 20:33--61, 1999.

\bibitem{daubechies2010iteratively}
I.~Daubechies, R.~DeVore, M.~Fornasier, and C.~S. G{\"u}nt{\"u}rk.
\newblock Iteratively reweighted least squares minimization for sparse
  recovery.
\newblock {\em Communications on Pure and Applied Mathematics}, 63(1):1--38,
  2010.

\bibitem{davidson2001local}
K.~R. Davidson and S.~J. Szarek.
\newblock Local operator theory, random matrices and {B}anach spaces.
\newblock {\em Handbook of the geometry of Banach spaces}, 1:317--366, 2001.

\bibitem{deshpande2013finding}
Y.~Deshpande and A.~Montanari.
\newblock Finding hidden cliques of size $\sqrt{N/e }$ in nearly linear time.
\newblock Technical Report 1304.7047, arXiv, 2013.

\bibitem{duda2012pattern}
R.~O. Duda, P.~E. Hart, and D.~G. Stork.
\newblock {\em Pattern classification}.
\newblock John Wiley \& Sons, 2012.

\bibitem{fevotte2009nonnegative}
C.~F{\'e}votte, N.~Bertin, and J.-L. Durrieu.
\newblock Nonnegative matrix factorization with the itakura-saito divergence:
  With application to music analysis.
\newblock {\em Neural computation}, 21(3):793--830, 2009.

\bibitem{fujishige2005submodular}
S.~Fujishige.
\newblock {\em Submodular Functions and Optimization}.
\newblock Elsevier, 2005.

\bibitem{goemans1995improved}
M.~X. Goemans and D.~P. Williamson.
\newblock Improved approximation algorithms for maximum cut and satisfiability
  problems using semidefinite programming.
\newblock {\em Journal of the ACM}, 42(6):1115--1145, 1995.

\bibitem{golub}
G.~H. Golub and C.~F. Van~Loan.
\newblock {\em Matrix computations}.
\newblock John Hopkins University Press, 1996.

\bibitem{zaidcg}
Z.~Harchaoui, A.~Juditsky, and A.~Nemirovski.
\newblock Conditional gradient algorithms for norm-regularized smooth convex
  optimization.
\newblock Technical Report 1302.2325, arXiv, 2013.

\bibitem{hyvarinen1999fast}
A.~Hyvarinen.
\newblock Fast and robust fixed-point algorithms for independent component
  analysis.
\newblock {\em IEEE Transactions on Neural Networks}, 10(3):626--634, 1999.

\bibitem{jaggi}
M.~Jaggi.
\newblock Revisiting {F}rank-{W}olfe: Projection-free sparse convex
  optimization.
\newblock In {\em Proceedings of the International Conference on Machine
  Learning (ICML)}, 2013.

\bibitem{jameson1987summing}
G.~J.~O. Jameson.
\newblock {\em Summing and nuclear norms in Banach space theory}.
\newblock Cambridge University Press, 1987.

\bibitem{journee2010low}
M.~Journ{\'e}e, F.~Bach, P.-A. Absil, and R.~Sepulchre.
\newblock Low-rank optimization on the cone of positive semidefinite matrices.
\newblock {\em SIAM Journal on Optimization}, 20(5):2327--2351, 2010.

\bibitem{journee2010generalized}
M.~Journ{\'e}e, Y.~Nesterov, P.~Richt{\'a}rik, and R.~Sepulchre.
\newblock Generalized power method for sparse principal component analysis.
\newblock {\em The Journal of Machine Learning Research}, 11:517--553, 2010.

\bibitem{pcacounter}
R.~Krauthgamer, B.~Nadler, and D.~Vilenchik.
\newblock Do semidefinite relaxations really solve sparse {PCA}?
\newblock Technical Report 1306.3690, ArXiv, 2013.

\bibitem{lee1999learning}
D.~D. Lee and S.~H. Seung.
\newblock Learning the parts of objects by non-negative matrix factorization.
\newblock {\em Nature}, 401(6755):788--791, 1999.

\bibitem{lee2010practical}
J.~Lee, B.~Recht, N.~Srebro, J.~Tropp, and R.~Salakhutdinov.
\newblock Practical large-scale optimization for max-norm regularization.
\newblock In {\em Advances in Neural Information Processing Systems (NIPS)},
  2010.

\bibitem{lewis2003mathematics}
A.~S. Lewis.
\newblock The mathematics of eigenvalue optimization.
\newblock {\em Mathematical Programming}, 97(1-2):155--176, 2003.

\bibitem{linial2007complexity}
N.~Linial, S.~Mendelson, G.~Schechtman, and A.~Shraibman.
\newblock Complexity measures of sign matrices.
\newblock {\em Combinatorica}, 27(4):439--463, 2007.

\bibitem{Mairal2010}
J.~Mairal, F.~Bach, J.~Ponce, and G.~Sapiro.
\newblock Online learning for matrix factorization and sparse coding.
\newblock {\em Journal of Machine Learning Research}, 11(1):19--60, 2010.

\bibitem{mallat1993matching}
S.~G. Mallat and Z.~Zhang.
\newblock Matching pursuits with time-frequency dictionaries.
\newblock {\em IEEE Transactions on Signal Processing}, 41(12):3397--3415,
  1993.

\bibitem{micchelli2013regularizers}
C.~A. Micchelli, J.~M. Morales, and M.~Pontil.
\newblock Regularizers for structured sparsity.
\newblock {\em Advances in Computational Mathematics}, 38(3):455--489, 2013.

\bibitem{murphy2012machine}
K.~P. Murphy.
\newblock {\em Machine learning: a probabilistic perspective}.
\newblock The MIT Press, 2012.

\bibitem{nesterov1998semidefinite}
Y.~Nesterov.
\newblock Semidefinite relaxation and nonconvex quadratic optimization.
\newblock {\em Optimization Methods and Software}, 9(1-3):141--160, 1998.

\bibitem{submodlp}
G.~Obozinski and F.~Bach.
\newblock Convex relaxation of combinatorial penalties.
\newblock Technical Report 00694765, HAL, 2012.

\bibitem{obozinski-joint}
G.~Obozinski, B.~Taskar, and M.~I. Jordan.
\newblock {Joint covariate selection and joint subspace selection for multiple
  classification problems}.
\newblock {\em Statistics and Computing}, 20(2):231--252, 2009.

\bibitem{recht2010guaranteed}
B.~Recht, M.~Fazel, and P.~A. Parrilo.
\newblock Guaranteed minimum-rank solutions of linear matrix equations via
  nuclear norm minimization.
\newblock {\em SIAM Review}, 52(3):471--501, 2010.

\bibitem{recht2012factoring}
B.~Recht, C.~Re, J.~Tropp, and V.~Bittorf.
\newblock Factoring nonnegative matrices with linear programs.
\newblock In {\em Advances in Neural Information Processing Systems (NIPS)},
  2012.

\bibitem{rockafellar97}
R.~T. Rockafellar.
\newblock {\em {Convex Analysis}}.
\newblock Princeton University Press, 1997.

\bibitem{srebro2005rank}
N.~Srebro and A.~Shraibman.
\newblock Rank, trace-norm and max-norm.
\newblock In {\em Learning Theory}, pages 545--560. Springer, 2005.

\bibitem{Stewart1990}
G.~W. Stewart and J.-G. Sun.
\newblock {\em {Matrix Perturbation Theory}}.
\newblock Academic Press, 1990.

\bibitem{xu2012robust}
H.~Xu, C.~Caramanis, and S.~Sanghavi.
\newblock Robust {PCA} via outlier pursuit.
\newblock {\em IEEE Transactions on Information Theory}, 58(5):3047--3064,
  2012.

\bibitem{zhang2010analysis}
T.~Zhang.
\newblock Analysis of multi-stage convex relaxation for sparse regularization.
\newblock {\em The Journal of Machine Learning Research}, 11:1081--1107, 2010.

\bibitem{zhang2012accelerated}
X.~Zhang, D.~Schuurmans, and Y.~Yu.
\newblock Accelerated training for matrix-norm regularization: A boosting
  approach.
\newblock In {\em Advances in Neural Information Processing Systems (NIPS)},
  2012.

\end{thebibliography}
